\documentclass[journal]{IEEEtran}

\usepackage{textcomp}
\usepackage{cite}
\usepackage{amsmath,amssymb,amsfonts}
\usepackage{algorithmic}
\usepackage{graphicx}
\usepackage{textcomp}
\usepackage{xcolor}
\usepackage{diagbox}
\usepackage{booktabs} 
\usepackage{multirow}
\usepackage{array}
\usepackage{cases}
\usepackage{makecell}
\usepackage{cases}
\usepackage{graphicx}
\usepackage{subfigure}
\usepackage{amsmath}
\usepackage{algorithm}             
\usepackage{algorithmic} 
\usepackage{textcomp}
\usepackage{amsfonts}
\usepackage{amssymb}
\usepackage{amsmath}
\usepackage{caption}
\usepackage{mathrsfs}
\usepackage{bm}
\usepackage{dsfont}
\usepackage{setspace}
\usepackage{setspace}
\usepackage{mathtools} 
\usepackage{enumitem}
\usepackage[colorlinks,
            linkcolor=blue,
            anchorcolor=blue,
            citecolor=blue]{hyperref}

\newtheorem{theorem}{\noindent \bf Theorem}
\newtheorem{lemma}{\noindent \bf Lemma}
\newtheorem{proposition}{\noindent \bf Proposition}
\newtheorem{definition}{\noindent \bf Definition}
\newtheorem{corollary}{\noindent \bf Corollary}
\newenvironment{remark}{{\noindent\it Remark:}\quad}{\hfill}
\newtheorem{assumption}{\noindent \it Assumption}
\newenvironment{proof}{{\noindent\it Proof:}\quad}{\hfill $\square$\par}
\def\BibTeX{{\rm B\kern-.05em{\sc i\kern-.025em b}\kern-.08em
    T\kern-.1667em\lower.7ex\hbox{E}\kern-.125emX}}
\usepackage[switch, columnwise]{lineno}

\begin{document}
\title{Distributed Stochastic Gradient Descent with Staleness: A Stochastic Delay Differential Equation Based Framework}
    
\author{  Siyuan~Yu,~Wei~Chen,~\IEEEmembership{Senior Member,~IEEE},~and~H. Vincent Poor,~\IEEEmembership{Life Fellow,~IEEE}

\thanks{Manuscript received 16 May 2024; revised 21 November 2024 and 24 February 2025; accepted 24 February 2025. This research was supported in part by National Natural Science Foundation of China and Research Grants Council (NSFC/RGC) Joint Research Scheme under Grant No. 62261160390/N HKUST656/22, in  part by National Natural Science Foundation of China under Grant No. 62471276, in part by a grant from Princeton Language and Intelligence, and in part by the U.S. National Science Foundation under Grant CCF-1908308 and Grant CNS-2128448. An earlier version of this paper has been presented in part at the 2024 IEEE International Conference on Communications (ICC) \cite{icc}. \emph{(Corresponding author: Wei Chen.)}} 

\thanks{Siyuan Yu and Wei Chen are with the Department of Electronic Engineering, the State Key Laboratory of Space Network and Communications, and the Beijing National Research Center for Information Science and Technology (BNRist), Tsinghua University, Beijing, 100084, China, e-mail: ysy20@mails.tsinghua.edu.cn, wchen@tsinghua.edu.cn. }
    
\thanks{H. Vincent Poor is with the Department of Electrical and Computer Engineering, Princeton University, New Jersey, 08544, USA, e-mail: poor@princeton.edu.}   
}

\maketitle
\thispagestyle{empty}

\begin{abstract}
Distributed stochastic gradient descent (SGD) has attracted considerable recent attention due to its potential for scaling computational resources, reducing training time, and helping protect user privacy in machine learning. {However, stragglers and limited bandwidth may induce random computational/communication delays, thereby severely hindering the learning process.} Therefore, how to accelerate asynchronous SGD (ASGD) by efficiently scheduling multiple workers is an important issue. In this paper, a unified framework is presented to analyze and optimize the convergence of {ASGD} based on stochastic delay differential equations (SDDEs) and the Poisson approximation of aggregated gradient arrivals. In particular, we present the run time and staleness of distributed SGD without a memorylessness assumption on the computation times. Given the learning rate, we reveal the relevant SDDE’s damping coefficient and its delay statistics, as functions of the number of activated clients, staleness threshold, the eigenvalues of the Hessian matrix of the objective function, and the overall computational/communication delay. The formulated SDDE allows us to present both the distributed SGD’s convergence condition and speed by calculating its characteristic roots, thereby optimizing the scheduling policies for asynchronous/event-triggered SGD. It is interestingly shown that increasing the number of activated workers does not necessarily accelerate distributed SGD due to staleness. Moreover, a small degree of staleness does not necessarily slow down the convergence, while a large degree of staleness will result in the divergence of distributed SGD. Numerical results demonstrate the potential of our SDDE framework, even in complex learning tasks with non-convex objective functions.
\end{abstract}

\begin{IEEEkeywords}
Distributed optimization, machine learning, stochastic gradient descent, gradient staleness, stochastic delay differential equations, asynchronous distributed optimization, straggler problem.
\end{IEEEkeywords}

\section{Introduction}
Deep neural networks (DNNs) have found widespread applications in diverse fields, including computer vision, natural language processing, and emerging areas like textual data analysis \cite{6g}, in which high-performance multi-GPU computing is enabled by GPU-oriented interconnects such as NVLink\cite{nvlink}. 
Computer networks, supported by communication protocols \cite{gallager, dailin}, often serve as the interconnecting infrastructure enabling efficient resource sharing among network nodes \cite{kurose}, which provides the essential framework for distributed AI systems to communicate and access resources.  {Stochastic gradient descent (SGD) serves as the fundamental framework for most of the existing deep learning algorithms, enabling models to learn from extensive databases and attain exceptional performance.}

{Parallelizing training can accelerate the overall training process. However, due to the need for synchronized updates, the completion time of each iteration is constrained by the slowest, so-called ``straggling" workers \cite{straggler_tsp, add_straggler}. This can negatively impact the convergence of the algorithm, as the presence of these stragglers leads to idle time, thereby reducing the potential speedup.}
In traditional synchronous learning, faster devices are forced to wait for slower ones due to the heterogeneity of the devices. To tackle this issue, asynchronous learning has been proposed, where each worker pushes gradients to the server in an event-triggered manner \cite{revisit, asaga,scaling,quad_sus, faster, yu_tcom, yu_tsp}. 
As an important advantage of {{asynchronous SGD} (ASGD)}, it has been shown to achieve better performance compared to synchronous SGD with respect to wall-clock time \cite{slow_and_stale}. Though {ASGD} eliminates waiting overhead, it suffers from higher convergence errors due to gradient staleness, which is induced by the computational delay and communication delay \cite{uot, Gc2023}.  As a result, stale gradients are usually penalized using a staleness-dependent learning rate in staleness-aware {ASGD} \cite{stale1, stale-aware, stale2, gap}. {Also, communication load can be reduced in distributed learning via event-triggered communication, referred to as event-triggered stochastic gradient descent \cite{sampling2, gupta, lazy, lazy2, israel}.
}

For the theoretical convergence analysis of SGD, classical bound-type results include \cite{nonlinearB} and \cite{ontheconver}. Existing bound-type results for SGD with stale-gradient include \cite{NBB01, final_add,
berkeleyDuchi,slow_and_stale, lihe, tac, add_major, add_major2, add_major3, mcmc, local}. The key obstacle to deal with in asynchronous learning is gradient staleness. In \cite{NBB01}, {ASGD} is shown to suffer an asymptotic penalty in convergence rate due to gradient staleness.  In the presence of gradient delays, Agarwal and Duchi demonstrate that asynchronous learning can achieve order-optimal convergence results in \cite{berkeleyDuchi}. As is  also presented in \cite{tac}, {ASGD} achieves similar convergence property as the centralized counterpart with bounded communication delay. To tackle this issue of gradient staleness, by carefully tuning the algorithm’s step size, asynchronous learning is shown to still converge the critical set in \cite{zhou2}. In addition to the classical bound type analysis, performance analysis based on continuous approximation by stochastic differential equation also attracted considerable attention in stochastic approximation literature \cite{Kushner,lizhiyuan}.
In \cite{SME}, the author uses tools from stochastic calculus and asymptotic analysis to provide a precise dynamical description of SGD and its variants, based on which adaptive learning rate policies are presented. In \cite{M2016},  the authors approximated SGD in terms of a multivariate Ornstein-Uhlenbeck (OU) process, where precise, albeit only distributional,
descriptions of the SGD dynamics are presented in complement to the classical bound type converge analysis. In addition, the authors conduct theoretical analysis on the convergence rates through the continuous approximation by stochastic differential delay equations in \cite{lihe}.

{Motivated by the straggler problem, and to provide non-asymptotic performance analysis and gain greater insights into the impact of gradient staleness on the convergence of distributed SGD, in this paper, a stochastic approximation-based approach is presented, in complement to the conventional bound type analysis.} Similar to \cite{M2016, SME, quad_sus}, we first focus on quadratic objective functions and then extend our discussion to general cases. Further, we aim to derive criteria for the parameter setup of learning rates, the total number of workers, and communication protocol in asynchronous distributed learning with gradient staleness.
In the literature, closely related works are \cite{slow_and_stale,uot} in terms of wall-clock run-time analysis, \cite{M2016, SME} in terms of stochastic differential equation approximation, and \cite{lihe} in terms of the connection between {ASGD} and stochastic delay differential equations (SDDEs).  {In this work, we extend the run-time analysis of SGD variants with exponential processing times \cite{slow_and_stale} to non-memoryless cases. We further extend the stochastic differential equation (SDE) approximation of SGD, developed in \cite{M2016, SME}, to an SDDE approximation that accounts for gradient staleness in ASGD. Furthermore, inspired by the stationary distribution analysis of SGD in \cite{M2016}, we perform a first hitting time analysis for SGD and a stationary distribution analysis for ASGD.
In addition to \cite{lihe}, we bridge the SDDE solution to ASGD dynamics with respect to the wall-clock time, and we also provide a detailed analysis of the characteristic roots of the SDDE, highlighting how factors such as the learning rate, the number of workers, and gradient staleness influence the ASGD convergence rate.}
\begin{figure*}
    \centering
    \includegraphics[width=6.2in]{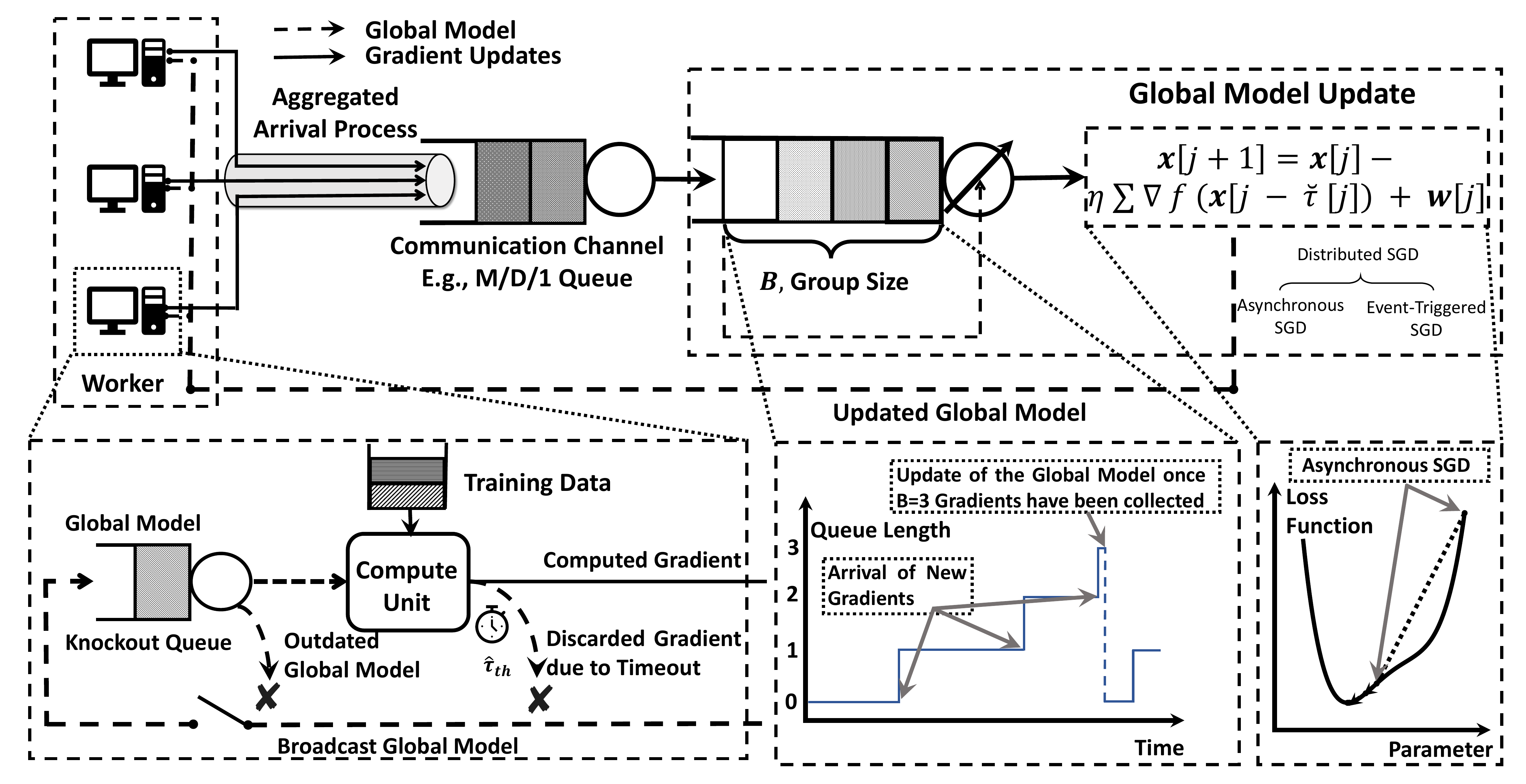}
    \caption{{System model.}}
    \label{fig:system_model}
\end{figure*}
{In this paper, we develop an SDDE-based approach, through which we bridge ASGD dynamics to the solution of an SDDE. Specifically, by presenting the characteristic root of the SDDE, we provide an insightful argument on the relationship between the convergence rate and gradient staleness, in complement to existing bound type analysis. Moving away from the memorylessness assumptions on computational times, the run-time and step staleness analysis in SDDE is further presented.  This approach enables us to reveal how factors such as learning rate, gradient staleness, number of workers, and communication protocol design influence ASGD performance, as well as presenting protocol design criteria.} 
Specifically, the contribution of this paper can be summarized as follows: 
\begin{itemize}[topsep  =0pt, partopsep = -30pt]
\item We present a unified framework to analyze and optimize the convergence process of {ASGD}  with gradient noise and staleness based on SDDEs. To deal with the non-exponentially distributed computation time, a Poisson approximation for the superposition of independent renewal processes is adopted, based on which we present the probability distribution of the update interval and the step staleness or discrete-time gradient staleness of {ASGD}. Based on the runtime and staleness analysis, we bridge the behavior of the {ASGD} algorithm and the solutions of the SDDE. 
In particular, the convergence time or rate of convergence of the SGD algorithm can be determined based on the first hitting time of the stochastic process or characteristic roots of the SDDE. 

\item The performance of the SGD algorithm with stale gradient is demonstrated to be closely related to the product of the step staleness, learning rate, and the 2-norm of the Hessian matrix. Our analysis indicates that a small degree of staleness could slightly accelerate the SGD algorithm, while {a large degree of staleness} could be harmful and result in the divergence of {ASGD}.  Our analyses align with the observations made by existing literature and complement classical bound-type convergence analyses for SGD with stale gradients. {It is further shown that gradient staleness slows down the escape of the ASGD algorithm from saddle points.} We further provide a theoretical analysis of the impact of uniformly distributed gradient staleness on the convergence rate of ASGD.

\item  {It is further revealed that reducing gradient staleness comes at the cost of a lower gradient update rate.} It is interestingly shown that, regardless of the specific distribution of computation time, {the expected step staleness} in {ASGD} without gradient dropout is only determined by the number of workers and the group size.  Also, it is shown that the presence of the gradient noise contributes to an increase in the average first hitting time. Under certain circumstances, the increase is proven to be proportional to the standard deviation of the Gaussian noise and the square root of the logarithm of the optimization variable’s dimension. {This work also shows that ASGD has a higher error floor compared to synchronous SGD. However, this issue can be alleviated by decreasing the level of asynchrony as ASGD nears the local minimum.}

\item
In practice, the optimal choice of worker number can be determined according to the learning rate and the $2$-norm of the Hessian matrix of the objective function.  An excessive number of activated workers does not necessarily accelerate {ASGD}. This is due to the fact that a small number of workers provides fewer gradients while a large number of workers leads to a greater degree of staleness. 
The performance of {ASGD} with an excess of workers can be improved by selecting an appropriate group size or discarding outdated gradients.  The limited bandwidth leads to an increased gradient staleness or a smaller update frequency {depending on whether the client starts computing a new gradient immediately after completing or transmitting the previous one}. In this case, an excessive number of workers can lead to a large communication delay due to network congestion, resulting in a degradation in the performance of SGD. 
Further, high-resolution quantization of the gradients with large gradient noise can lead to an overload of network bandwidth, thereby being inefficient due to the waste of communication resources.

{
\item ASGD offers an effective solution to address the straggler problem arising in synchronous SGD.  While ASGD may be hindered by stale gradients, this paper demonstrates that, through the careful selection of the number of workers or by designing appropriate protocols, step staleness can be effectively managed. Also, it is possible to leverage stale gradients to accelerate the convergence of ASGD. Additionally, ASGD benefits from a higher gradient arrival rate compared to synchronous SGD, which can lead to better convergence rates in terms of wall-clock time.}

\end{itemize}

The rest of this paper is organized as follows.  Section II presents the system model and the problem formulation. Section III presents the performance analysis for SGD with gradient noise and staleness through a stochastic delay differential  equation-based approach. Section IV presents the run-time, staleness, and protocol design criteria for distributed SGD.  
Simulation results are provided in Section V. Finally, conclusions are drawn in Section VI.

\section{System Model}
\begin{table*}[t]
\centering
\setlength{\tabcolsep}{0.5mm}
\begin{tabular}{|c|l||c|l|}
\hline  Notation &  {~~~~~~~~~~~Description } &  Notation  &   {~~~~~~~~~~~~~~~~~~Description} \\

\hline $\hat {\pmb x}(t)/\breve {\pmb x}[j]$ & The objective variable in continuous/discrete-time  & $K$ & The number of total workers\\

\hline  $\hat \tau$& {Gradient delay or continuous-time gradient staleness}   & $\pmb w$ & The gradient noise\\
\hline $\breve \tau$   & {Step staleness or discrete-time gradient staleness}   & $\pmb B(t)$ &Standard Wiener process \\

\hline $\tau$ & Gradient staleness in SDDE  & $S_j$  & The time at which the global model is updated\\

\hline $f(\cdot)$ & The objective function  &$\eta$ &The learning rate or step size 
\\
\hline $\nabla f(\cdot)$ & The gradient function   & $T_\kappa$  & The computational delay of the $\kappa$th client\\

\hline $N_p(t)$&  The aggregated gradient arrival process   & $H$  & The first hitting time\\

\hline $\pmb V$ & Hessian matrix of the objective function   & $N$  & The dimension of the objective function\\

\hline $v$ & The eigenvalues of $\pmb V$   & $\lambda$  &The characteristic roots of the SDDE\\

\hline  $W_k(\cdot)$& The Lambert W function  &  $\varpi$ & A constant number\\
\hline  $\delta$ & A small positive number   & $\mu_Q$  & The service rate of the bandwidth-limited channel\\
\hline $\nu$  & The intensity of the arrival process   &  $B$ & The group size\\
\hline

\end{tabular}
\caption{\label{tab:table_var}{Table for Notations.}}

\end{table*}

In this paper, we are interested in centralized distributed learning involving  gradient noise and staleness. {Specifically, there is a single central parameter server connected with $K$ parallel clients, or equivalently workers, via a communication medium} as illustrated in Fig. \ref{fig:system_model}.  In the conventional synchronous setting, the global model evolves based on gradients from all the $K$ workers at each iteration. To mitigate stragglers in synchronous SGD, ASGD has been introduced.  In the asynchronous setup, each worker can operate independently, fetching the global model and updating its computed gradient. Furthermore, to alleviate the communication overload in distributed learning, event-triggered SGD is presented where the worker {pushes} its gradient to the central server until the difference between the previous gradient and the current one exceeds a given triggering threshold. We introduce the implementation details of {ASGD} in Subsection \ref{suba}, and provide a corresponding stochastic approximation. {In Subsection \ref{subb}, key mathematical concepts used in this paper are presented.}

\subsection{Asynchronous Distributed Optimization}\label{suba}
{Let us first introduce the objective function in distributed SGD. Let $\pmb{x} \in \mathbb{R}^N$ be the parameters of the model to be trained. Let $f_i(\pmb{x})$ be the loss incurred by $\pmb x$ at the $i^{\text {th }}$ data point $\xi_i\sim \mathcal{D}$, in which $\mathcal D$ is the set of the training data. The objective is to minimize the generalization error 
{
\begin{equation}\label{add_aq}
    f(\pmb{x})=\mathbb{E}_{\xi_i \sim \mathcal D}\left\{f_i(\pmb{x})\right\}, 
\end{equation}}where $\mathbb{E}\{\cdot\}$ denotes expectation.}
To mitigate stragglers in conventional synchronous distributed learning, asynchronous distributed learning has emerged as an alternative approach. In the asynchronous setting, each worker operates independently of others, fetching the global model and updating its computed gradient asynchronously. The process unfolds as follows:{ Firstly, a worker fetches from the central server the most up-to-date parameters of the model to process the current mini-batch.} Subsequently, it computes gradients of the loss function with respect to these parameters. These gradients are then sent to the central server, which then updates the model accordingly. Simultaneously, upon completing the computation, it fetches the updated global model right away, based on which it computes new gradients.\footnote{In this paper, we assume that the downlink communication cost is negligible. Also, the time it takes for the central server to update the global model is also assumed to be negligible since the update is based on computed gradients. In this way, each worker could fetch the updated global model and start the computation of a new gradient once it finishes the computation of the old one.}  In the asynchronous setup,  the gradient utilized to update the global model is often computed based on outdated parameters, leading to gradient staleness, which can impede the convergence of {ASGD}. To address this challenge, we explore the following two methods aimed at reducing gradient staleness in {ASGD}.

First, in $B$-ASGD, the central server updates the global model once $B$ gradients are collected since the last update, while the workers who fail to {push} the gradient in the current iteration can continue their computation. In this context, $B$ is referred to as the group size. In particular, the $B$ clients, based on whose gradients the global model updates, start the computation of a new gradient based on the updated global model, to reduce the gradient staleness. Note that the $B$ clients could also start their computation of new gradients right after their gradients are received by the central server. In this case, the {ASGD} algorithm is termed $B$-batch-ASGD. In this paper, we generally focus on $B$-ASGD unless otherwise stated.   Note that {ASGD} reduces to synchronous SGD if $B = K$.\footnote{In this paper, B-ASGD follows the same algorithm as K-async SGD described in \cite{slow_and_stale}, with additional features such as gradient dropout to manage gradient staleness. Further, as our study uses different analytical methods and assumptions, we use distinct terminology to differentiate the two approaches.} 

Second, the workers could drop the current computing gradient to avoid a large degree of gradient staleness. 
In particular, let $\hat \tau_\text{th}$ denote the continuous-time gradient staleness threshold. 
More specifically, during the computation of a gradient, if the time since a worker last fetched the global model exceeds $\hat \tau_{th}$, the worker stops the computation of the current gradient, fetches the current model, and start the computation of a new gradient. In fact, the worker can also drop outdated gradients according to discrete-time staleness or step staleness. In particular, if the number of updates taken by the central server exceeds a pre-given threshold $\breve \tau_\text{th}$ since the client last {fetched} the global model, the client halts the current training and starts the computation of the new gradient based on the current global model. By this means, {ASGD} reduces the gradient delay at the expense of a smaller{ {gradient update frequency}}. In this context, we call it pure {ASGD} when $B=1$ and $\hat \tau_\text{th} = \infty$ or $\breve \tau_\text{th} = \infty$. 

In the following, we present the continuous-time
update rule for {ASGD}. Let $\mathcal K(t)$ be the set of workers whose gradients arrive at the parameter server until time $t$ since the global model is last updated. In $B$-ASGD, it holds that $\mathcal K(t^+) = \Phi$ if $\|\{{\mathcal K(t)}\}\|=B$, in which $t^+$ means a time that is infinitesimally but noticeably later than the time $t$. In this context, $\mathds{1}\left\{\|\mathcal K(t)\| = B\right\}$ is $1$ at the time that the central server collects $B$ gradients and updates the global model and $0$ otherwise, in which $\mathds{1}\left\{\cdot\right\}$ denotes the indicator function and $\|\cdot\|$ denotes the cardinality of a set. Let $\hat{\pmb x}(\cdot)\in \mathbb R^N$ denote the model parameter with respect to the wall-clock time, where the wall-clock time refers to the actual time
elapsed and $N$ is the dimension of the parameter.  
In this way, the update rule for {ASGD} with respect to the wall-clock time is given by
\begin{equation}
\begin{aligned}\label{continuous_asyn}
    &\hat {\pmb x}(t^+) = \hat {\pmb x}(t) - \\
    &\sum_{\kappa \in  \mathcal  K(t)}\eta \Big[\nabla f \left(\hat {\pmb x}({t - \hat \tau^\kappa(t)}\right)) +\pmb w^\kappa(t)\Big]\times \mathds{1}\left\{\|\mathcal K(t)\| = B\right\},
\end{aligned}
\end{equation}
in which $\eta$ is the learning rate,  $f(\cdot)$ is the objective function, $\nabla f(\cdot)$ is the gradient function of the objective function, $\pmb w^\kappa(\cdot)$ is the gradient noise corresponds to Worker $\kappa$'s gradient, and $\hat \tau(\cdot)$ is the gradient delay or the continuous-time gradient staleness. More specifically, $\eta$ is a constant with a fixed learning rate. {With a staleness-aware learning rate, $\eta$ can be a function of the gradient delay, or equivalently, continuous-time gradient staleness.  Alternatively, $\eta$ can also be a function of the step staleness, or equivalently, discrete-time gradient staleness.}. In addition, $\hat \tau(\cdot)$ refers to the time between a worker fetches the global model and {pushes} the computed gradient to the central server, and $\hat \tau^\kappa(t)$ denotes the gradient delay of Worker $\kappa$'s gradient that is delivered to the central server at time $t$. In practice, the gradient delay is {jointly} determined by the computation/communication resources, the group size $B$,  and the continuous-time or discrete-time staleness threshold $\hat \tau_{th}$ or $\breve \tau_{th}$. Also, $\pmb w^\kappa(\cdot)$ is determined by the mini-batch size of each client.  Here, we adopt the general assumption of the Gaussianity of the gradient noise, which arises from the central limit theorem \cite{noGau}. Also, we assume that the Gaussian distribution is isotropic and the variances of the gradient noise across each worker are identical. Next, we present the following two assumptions on the objective function and the gradient noise.{
\begin{assumption}\label{ass_twice_diff}
The objective function $f(x)$ is assumed to be twice-differentiable at $\pmb x \in \mathbb R ^N$, in which $\pmb x$ represents the parameters of the model to be trained as defined in Eq. (\ref{add_aq}).
\end{assumption}}
\begin{assumption}\label{ass_noi}
    The gradient noise $\pmb w$ is assumed to follow an isotropic Gaussian distribution.
\end{assumption}

{
Under Assumption \ref{ass_twice_diff}, we approximate the twice-differentiable objective function \( f(\pmb x) \) using Taylor expansion. More precisely, we have $f(\pmb{x}) = f(\pmb{x}_0)  \nonumber +(\pmb{x} - \pmb{x}_0)^T \nabla f(\pmb{x}_0) + (\pmb{x} - \pmb{x}_0)^T \pmb{H}(\pmb{x}_0) (\pmb{x} - \pmb{x}_0) + \mathcal O(\delta^3), $ by Taylor expansion under Assumption \ref{ass_twice_diff} for \( \|\pmb{x} - \pmb{x}_0\| < \delta \) with some small positive \( \delta \)  \cite[Chapter 3.3]{add_aq2}. The assumption of twice differentiability is made to facilitate subsequent analysis, particularly for discussing the impact of staleness on convergence, based on the positive definiteness of the Hessian matrix.
}

Next, we aim to bridge the update rule of asynchronous learning in continuous-time (\ref{continuous_asyn}) and its discrete-time counterpart. 
Let $0= S_0 < S_1 < S_2 <...< S_{J^*}$ denote the times at which the central parameter updates the global model, in which SGD meets its stopping criteria after $J^*$ iterations.  In this setting, the discrete-time counterpart of $\hat {\pmb x}(S_j)$ is denoted by $\breve {\pmb x}[j]$. More specifically, it holds that $J^* = \min\limits_j \{j\in \mathbb N ^+, \left\|\breve{\pmb x}[j] - \pmb x ^*\right\| <\delta\}$, where $\pmb x^*$ is the optimal solution of the SGD algorithm and $\delta>0$ is a pre-given threshold. In literature, the probability distribution of the update interval, denoted by $I_j = S_j - S_{j-1}$, is presented under the assumption of exponentially distributed computational time \cite{slow_and_stale}.  Herein, we will characterize the distribution of the update interval $I$ through a Poisson approximation with non-memoryless computational delay in Subsection \ref{sub_toadd}.     Let $N_a(t) = \{\max i| S_i \le t\}$ and $N_a(t_1,t_2) = N_a(t_2) - N_a(t_1)$ iterations. More specifically, $N_a(t)$ is the number of global updates taken before time $t$.
In this paper, the wall-clock time refers to the actual time elapsed regarding the training process, and the discrete-time counterpart of which is termed iteration or epoch.  
Clearly, the wall-clock time corresponding to the $j$th update of the global model is $S_j$. Correspondingly, $\pmb w^\kappa[\cdot]$ indexes the discrete-time gradient noise of Client $\kappa$.   Also, $\left\|\mathcal K(S_j)\right\| = B$. Let $\mathcal K[j] = \mathcal K(S_j)$ be the set of clients that triggers the $j$th update of the global model.
Thereby, the update rule of asynchronous learning in discrete-time corresponding to (\ref{continuous_asyn}) is given by
\begin{equation}\label{afd}
    \breve {\pmb x}[j+1] = \breve {\pmb x}[j] - \sum_{\kappa \in \mathcal K[j]} \eta\left(\nabla f (\breve {\pmb x}\left[{j - \breve \tau_B^{\kappa}[j]}\right]) + {\pmb w}^{\kappa}[j]  \right),
\end{equation}
in which $\breve \tau_B^{\kappa}[j]$ is the iterations taken by the global model between Worker $\kappa$ fetches the global model and finishes the computation of the gradient that contributes to the $j$th update of the global model. In $B$-ASGD, the relationship between the step staleness or discrete-time staleness $\breve \tau^{\kappa}[j]$  and $\breve \tau_B^{\kappa}[j]$ is given by $\breve \tau^{\kappa}[j]=  \breve \tau_B^{\kappa}[j] \times B$, since the global model is updated based on $B$ gradients at each iteration.
The relationship between the discrete-time gradient staleness $\breve \tau^\kappa[j]$ and the continuous-time staleness $\hat \tau^\kappa(t)$ in $B$-ASGD {is} given by
\begin{equation}
    \breve \tau^\kappa[j] = N_a(S_j - \hat \tau^\kappa(S_j), S_{j-1})\times B.
\end{equation}
Let the step staleness of the arrived gradients be denoted by the random variable $\breve \tau$. 
Intuitively, a larger $\mathbb E\{\breve \tau\}$  corresponds to an increased training error, as discussed in \cite{stale-aware}. 
Also, it is noteworthy that a larger gradient delay does not necessarily result in a greater step staleness. {Note that continuous-time gradient staleness and gradient delay refer to the same concept: the time interval between when a worker fetches the global model and when the computed gradient is pushed back to the central server \cite{major_add}. This value is a non-negative real number. In addition, ``discrete-time gradient staleness" and ``step staleness" refer to the same concept \cite{faster}. This term describes the number of iterations taken by the global model between when a worker fetches the model and when it completes the computation of the gradient. It is a non-negative integer.}

Let $\pmb x(\cdot)$ denote the continuous-approximation of $\hat{\pmb x}(\cdot)$ with respect to the stochastic differential equation.
{In this way, the dynamics of {ASGD} can be written as an Euler–Maruyama approximation of the following stochastic delay differential equation:}
\begin{equation}\label{con_adl}
    \mathrm{d}{\pmb x}(t)  = - \nabla f({\pmb x}(t - \tau(t))) \mathrm{d}t  + \sigma \mathrm{d}\pmb B(t),
\end{equation}
in which $\nabla f(\pmb x)$ is the gradient of the objective function $f(\cdot)$ at $\pmb x$, $\tau(t)$ is the stochastic process characterizing
the gradient staleness in the SDDE, $\pmb B(t)$ or $B(t)$ is an $N$-dimensional or one-dimensional standard Brownian motion or Wiener process, and $\sigma >0 $ is the standard variance of the Wiener process, {in which a standard $N$-dimensional Wiener process is a vector-valued stochastic process
$\pmb B(t) = \left(B_1(t), B_2(t), \ldots, B_N(t)\right)$,
where each $B_i(t)$ is an independent standard one-dimensional Wiener process for $i = 1, 2, \ldots, N$.}\footnote{To be rigorous, given the learning rate $\eta$ and the standard variance of $\pmb w[\cdot]$, denoted by $\breve \sigma$, the relationship between the standard variance of {Wiener process} in (\ref{con_adl}) and the Gaussian noise in (\ref{afd}) is  $\sigma = \frac{1}{\sqrt{\eta}}\breve \sigma$. }

Let $\breve \tau^{(n)}$ denote the step staleness of the $n$th gradient that is delivered to the central server. 
With a fixed learning rate $\eta$, the gradient staleness in SDDE is given by $\tau(t) = \eta \breve \tau^{(n)}$ for $t\in ((n-1)\eta, n\eta)$.\footnote{A careful reader may notice that if the central server performs updates of the global model based on multiple gradients at one time, it could introduce subtle differences in gradient staleness in SDDE, as later-added gradients may suffer a greater staleness. In this paper, we ignore this effect due to its minor impact. Also, the approximation error of (\ref{con_adl}) is investigated in \cite{lihe} and is therefore beyond the scope of this paper.}   The relationship between $\hat {\pmb x}(\cdot)$ and its continuous-time approximation $\pmb x (\cdot)$ is $\hat{ \pmb x}(t) \approx \pmb x(\eta N_a(t)B)$.

In Section \ref{sec3}, we aim to reveal the relationship between gradient staleness and noise and the convergence rate of the SDDE. In Section \ref{sec_pra_dis}, we shall reveal how the protocol parameters $K$, $B$, and the computational/communication delays determine $N_a(t)$ and the step staleness $\breve \tau$, as well as the gradient dropout schemes that can be adopted to trade update frequency for a smaller step staleness. According to the SDDE approximation, we bridge the convergence rate of {ASGD}  and the learning rate, the objective function, and the protocol parameters, based on which the protocol design criteria are further revealed.

\subsection{Event-Triggered Distributed Optimization}\label{subb}

In this subsection, we present the implementation of ET-SDG. In distributed SGD with limited bandwidth, event-triggered SGD can be adopted to effectively reduce bandwidth usage. In particular, different from {ASGD}, time is partitioned into fixed-length time slots in the ET-SDG setup. Let $T_s$ denote the time slot length. The central server performs updates of the model at wall-clock time $jT_s$ for $j \in \mathbb N ^+$. Further, the communication between the workers and the central server follows an event-triggered manner. In this scheme, each worker communicates its gradient to the server only upon a certain event is triggered. More specifically, the triggering condition is that the difference between the current gradient and the last updated gradient is greater than a pre-given threshold.

In particular, the central server broadcasts the current model to all the workers at each iteration. Based on the received model and the local dataset, each worker computes a gradient. Further, it is assumed that each worker could finish the computation of the gradient within a timeslot. More specifically, the gradient computed by Worker $\kappa$ at iteration $j$ is denoted by ${\pmb g}^\kappa[j] =\nabla f (\breve {\pmb x}[j]) + \pmb w^\kappa[j] $. Note that the workers skip transmission of the current gradient at the epoch when the current gradient and the last updated gradient do not differ too much \cite{lazy}. More specifically, the triggering criterion is given by 
    $\left\|{\pmb g}^\kappa[j] -   \hat {\pmb g}^\kappa[j^-]\right\|_2 \ge \xi$
in which ${\pmb g}^\kappa[j]$ is the newly computed gradient, $ \hat {\pmb g}^\kappa[j^-]$ is Worker $\kappa$'s last updated gradient before time $j^-$, $\left\| \cdot \right\|_2$ denotes the 2-norm, and $\xi$ is the threshold in event-triggered SGD. Further, let
$
\hat {\pmb g}^\kappa[j] ={\pmb g}^\kappa[j - \breve\tau^\kappa[j]],
$
in which $\breve\tau^\kappa[j]$ is the gradient staleness that arises in event-triggered SGD and $\breve\tau^\kappa[j] = 0$ if Worker $\kappa$ {pushes} its gradient at time $j$.
In this way, the update rule of the global model in event-triggered SGD in discrete-time is $ \breve {\pmb x}   [j+1] = \breve{\pmb  x}[j] -\frac{\eta}{K} \sum_{\kappa = 1}^{K} \hat{\pmb  g}^\kappa[j - \breve\tau^{\kappa}[j]].$ Slightly different from the {ASGD} setting, the learning rate is $\eta/K$ with $K$ workers in event-triggered SGD.
Correspondingly, the update rule of event-triggered SGD can also be viewed as the discrete-time version of the continuous-time process (\ref{con_adl}), in which $\tau(t) = \eta \breve \tau^{\kappa}[j]$ for $t\in ((j-1)\eta + (\kappa -1)\eta/K, (j-1)\eta + \kappa \eta/K)$ and $\breve {\pmb x}[j]\approx \pmb x(j\eta )$ in event-triggered SGD.

In Section \ref{sec3}, we aim to reveal the relationship between gradient staleness and noise and the convergence rate of the SDDE. In Section \ref{sec_pra_dis}, we shall reveal how the protocol parameters $K$, $B$, and the computational/communication delays determine $N_a(t)$ and the step staleness $\breve \tau$, as well as the gradient dropout schemes that can be adopted to trade update frequency for a smaller step staleness. According to the SDDE approximation, we bridge the convergence rate of {ASGD}  and the learning rate, the objective function, and the protocol parameters, based on which the protocol design criteria are further revealed.

\subsection{Definitions of the Key Mathematical Concepts}\label{subc}

In this subsection, we present definitions of the key mathematical concepts of this paper. Let us first introduce the concept of the OU process.
\begin{definition}
     A one-dimensional OU process $O(t)$ is the solution to the stochastic differential equation
     \begin{equation}
         {\rm d}O(t) = \alpha (\beta  - O(t)){\rm d}t  +  \sigma {\rm d}B(t), \text{ for } t > 0,
    \end{equation}
    where $\alpha$, $\beta$, and $\sigma \ge 0$ are real constants, and $B(t)$ is a standard Wiener process.
 \end{definition} 

 A one-dimensional OU process $O(t)$ can also be defined in terms of a stochastic integral as
\begin{align}
    &O(t) = \beta(1 - \exp(-\alpha t))
     + 
     \\&\sigma  \exp(-\alpha t)\int_0^t \exp(\alpha s){\rm d}B(s) + O(0)\exp(-\alpha t).
\end{align}
where $O(0)$ is a random variable denoting the initial state. Further, conditional on $O(0)$ where $\mathbb E\{O(0)\} < \infty$, $O(t)$  is a Gaussian random variable with mean $\mathbb E\{O(t)\}  = \beta (1 - \exp(-\alpha t)) + \exp(-\alpha t)\mathbb E\{O(0)\}$. In this paper, our main focus is on the case that $\beta  = 0$. In this case, it holds that  
\begin{equation}
    O(t) = \sigma  \exp(-\alpha t)\int_0^t \exp(\alpha s){\rm d}B(s) + O(0)\exp(-\alpha t).
\end{equation}
which is a Gaussian random variable 
with mean $\mathbb E\{O(t)\}  = \exp(-\alpha t)\mathbb E\{O(0)\}$ and variance $\text{Var}\{O(t)\}  =\frac{\sigma^2}{2\alpha}(1 -\exp(-2\alpha t))$.

The stopping criterion of SGD is usually of the form $\|\nabla f(x)\|_2 \le \delta$, where $\delta$ is small and positive \cite{bv_book}. In this case, the training time of SGD can be characterized by the first hitting time of the objective variable to the set $S= \{x: \|\nabla f(x) \|_2\le \delta\}$. In the following, let us present the definition of the first hitting time.
\begin{definition}
    Let $\{X(t),{t\ge 0} \}$ denote a stochastic process in an $N$-dimensional spatial domain $\Omega \subseteq \mathbb R^{N}$, the first hitting time to a target set $\Omega_t \subset \Omega$ is defined as 
    \begin{equation}
        H := \inf \{t\ge 0: X(t)\in \Omega_t\}.
    \end{equation}
\end{definition}

For Brownian motion with drift, we have the following lemma according to \cite{FirstPass}.
\begin{lemma}\label{lemma_averagespeed}
The expectation of the first hitting time $H = \inf \{ t> 0 : x_1\le  x(t)\le  x_1 +\mathrm d x, x(0) = x_0 \}$ of the one-dimensional stochastic process $\mathrm{d} x(t)  =  s(x) \mathrm d t + \sigma(x) \mathrm{d}B(t)$ is $\mathbb{E}\{H \} = \int_{x_0}^{x_1}\frac{1}{s(x)}\mathrm d x,$
in which $s(x)$ is the drift velocity.
\end{lemma}
For example, the fact that the expected hitting time of a given level by a
standard Wiener process is infinite is a direct corollary of Lemma \ref{lemma_averagespeed}. 
For convenience, we summarize the definitions of important variables
in Table \ref{tab:table_var}.

\section{A Unified Framework for Convergence Analysis of Distributed SGD}\label{sec3}
In this section, we present theoretical performance analyses for the relationship between {{gradient staleness}} and noise and the convergence of distributed SGD through an SDDE-based approach.
To obtain analytical results and find more insights, we focus on the solvable cases in this section, as was done in, e.g., \cite{SME, M2016}.  More specifically, in this section, our analysis {{focuses}} on {{the}} quadratic model \cite{toadd11, toadd12, toadd13}. Specifically, the objective function takes the form
 $f(\pmb x) = \frac{1}{2}(\pmb x- \pmb x^*)^T \pmb V (\pmb x - \pmb x^*) + \pmb q^T( \pmb x - \pmb x^*) + \pmb c,$
in which $\pmb V \in \mathbb R^{N\times N}$ is a {symmetric matrix, $\pmb x \in \mathbb R^{N}$ is the objective variable, $\pmb q, \pmb c$, and $\pmb x^*\in \mathbb R^{N}$ are constants}, and $[\cdot]^T$ denotes transpose. {
In the literature, including \cite{SME, M2016, add_boyd}, which uses differential equations to analyze optimization methods for deeper understanding, quadratic models are frequently employed to derive closed-form solutions.}

\subsection{A Unified Framework Based on SDDE}
{
In this subsection, we first present the following discussion on the objective function.
Note that $\pmb V$ is a symmetric matrix. Therefore, it holds that $\pmb V = \pmb U \tilde {\pmb V}  \pmb U^T$ in which $\tilde {\pmb V}$ is a diagonal matrix and $\pmb U$ is a unitary matrix by eigendecomposition. Hence, by defining  $\tilde {\pmb x} = \pmb U^T (\pmb x-\pmb x^*)$ and $\tilde q = \pmb q^T \pmb U$, the objective function can be rewritten as $f\left(\tilde {\pmb  x}\right) = \frac{1}{2}\tilde {\pmb x}^T \tilde {\pmb V} \tilde{\pmb x} + \tilde {\pmb q}^T \tilde{ \pmb x} + \pmb c$. where $\tilde {\pmb V}$ is a diagonal matrix. By omitting the tildes for brevity, the objective function can be formulated as 
   $f(\pmb x) = \frac{1}{2}\pmb x^T \pmb V \pmb x + \pmb q^T \pmb x + \pmb c$.
By this means, we assume the symmetric matrix $\pmb V$ is a diagonal matrix given by $\pmb V = \text{diag}(v_1, v_2, ..., v_N)$. without loss of generality.
Further, the continuous-time approximation of distributed SGD with gradient staleness (\ref{con_adl}) can be represented as:
\begin{equation}\label{ou_mp}
   \mathrm{d} {\pmb x}(t)  = - ({{\pmb V}{\pmb x}(t - \tau(t))  + \pmb q})\mathrm{d}t  + \sigma \mathrm{d}\pmb B(t), 
\end{equation}
or 
\begin{equation}\label{ou_p_tmp}
    \mathrm{d}x_i(t)  = - (v_ix_i(t - \tau(t)) + q_i)\mathrm{d}t  + \sigma \mathrm{d}B(t),
\end{equation}
for $i = 1, 2, ..., N$, where $q_i$ is the $i$th element of $\pmb q$.\
Further, by introducing $\bar x_i(t) = x_i(t) + \frac{q_i}{v_i}$ when $v_i\ne 0$, Eq. (\ref{ou_p_tmp}) can be rewritten as 
\begin{equation}\label{ou_p}
       \mathrm{d}\bar x_i(t)  = - v_i\bar x_i(t - \tau(t))\mathrm{d}t  + \sigma \mathrm{d}B(t),\text{ for } v_i \ne 0.
\end{equation}
Therefore, for the sake of discussion, we make the assumption that $q_i = 0$ for any $i$ when $v_i \ne 0$ without any loss of generality.}

{\begin{remark}
Similar to most existing work in the literature on the continuous approximation analysis of SGD \cite{revisit,add_dynamic, M2016, SME}, we focus on quadratic objective functions, as they allow for closed-form analysis of training time, convergence rate, and the stationary distribution. For example, in \cite[Section 7.1.1]{bishop}, it is noted that insights into the optimization problem and the various techniques for solving it can be gained by considering a local quadratic approximation to the error function.  

While this method can be extended to twice-differentiable functions \( f(\pmb x)\) to analyze the local behavior of SGD, we focus on approximating \( f(\pmb x) \) around a specific point \( \pmb{x}_0 \) within its domain. For small neighborhoods around \(\pmb  x_0 \), we approximate \( f(\pmb x) \) using a quadratic model function derived from Taylor expansion, as outlined by Assumption \ref{ass_twice_diff}. For example, the discussion on the quadratic model is particularly relevant when SGD approaches a local or global minimum. In this case, \( \pmb{V} \) is a positive-defined matrix, meaning that all of its eigenvalues are greater than zero, and it holds that \( \pmb{q} = 0 \). On the other hand, when some eigenvalues of \( \pmb{V} \) are negative, this corresponds to SGD behavior near a saddle point. 
This approximation is also extensively utilized in the literature, particularly in trust-region methods \cite{trustRegion ,trust_ppo}. In this approach, the objective function is decomposed into a sequence of smaller subproblems, each of which is approximated by a quadratic model, thereby facilitating more efficient optimization by focusing on local regions of the solution space.
\end{remark}
}

When $\sigma = 0$, implying that each gradient is computed based on the entire training dataset {{such that}} the gradient noise is negligible, Eq. (\ref{ou_p}) reduces to {the deterministic differential equation}: $\mathrm{d}x_i(t)  = - v_i x_i(t - \tau(t)) \mathrm{d}t$.
In addition, when $\tau(t) = 0$, which corresponds to synchronous SGD, we obtain 
$\mathrm{d}x_i(t)  = - v_i x_i(t) \mathrm{d}t  + \sigma \mathrm{d}B(t)$, {which corresponds to an OU process}. Also, it is observed that the OU process with delay (\ref{ou_mp}) along each eigenvector direction is independent, allowing us to first focus on the one-dimensional case (\ref{ou_p}) then extend our discussion to the multi-dimensional case. 

Next, we demonstrate that the effects of the gradient noise and the gradient staleness can be separately considered taking the case that $\tau$ is a constant as an example.\footnote{Note that the gradient staleness $\tau$ in SDDE is approximately a constant the variance of computation delay of each worker is small with a large worker number in {ASGD}, as will be demonstrated in Section \ref{sec_pra_dis}.} In this case, the analysis of the behavior of SGD can be based on viewing SGD as a discretization of the following associated stochastic delay differential equation, i.e.,
\begin{equation}\label{delay_sde}
\begin{aligned}
        &\mathrm{d}x(t)  = - {v} x(t-\tau)\mathrm{d}t   + \sigma \mathrm{d}B(t),\\
        &x(t) = \varphi(t),~~ t\in[-\tau, 0],
\end{aligned}
\end{equation}
in which $\varphi(t)$ is a given function on $[-\tau, 0]$.  
Consider the following continuous deterministic differential equation corresponding to Eq. (\ref{delay_sde}), i.e.,
\begin{equation}\label{determine_function}
\begin{aligned}
        & \mathrm{d}\tilde x(t)  = - {v} \tilde  x(t-\tau) \mathrm{d}t, \\
        & \tilde  x(t) = \varphi(t),~~ t\in[-\tau, 0].
\end{aligned}
\end{equation}
Let $\tilde  x_d(t)$ denote a solution to Eq. (\ref{determine_function}). Then, the solution to Eq. (\ref{delay_sde}) can be given by $ x_s(t) = \tilde  x_d(t) + \sigma\int_0^t x_0(t-s)\mathrm dB(t),$
in which $x_0(t) = \sum_{k = 0}^{\lfloor \frac{t}{\tau}\rfloor}\left[\frac{{v}^k}{k!}(t-k\tau)\right]$ in which $t\in (0, \infty)$. In addition, we have $\tilde  x_d(t) = x_0(t)\varphi(0) + {v}\int_{-\tau}^{0} x_0(t-s-\tau)\varphi(s)\mathrm ds$.
It can be noticed that the difference between $x_s(t)$ and $\tilde  x_d(t)$ is only given by a Gaussian distributed random variable, the variance of which is {jointly determined} by the gradient staleness $\tau$ and the standard variance of the {Wiener process} $\sigma$.

In the following discussions, we will present performance analyses of the convergence rate of distributed SGD, considering the effects of gradient noise and gradient staleness, through the SDDE-based approach \cite{india, nonconvex2017}.

{
\begin{remark}
    In this paper, we adopt the general assumption that the gradient noise follows a Gaussian distribution in Assumption \ref{ass_noi}. This Gaussian assumption is widely adopted in the literature, as demonstrated in works such as \cite{lihe, M2016,gaussian1, gaussian2}. However, the gradient noise can exhibit non-Gaussian characteristics, especially with small minibatch sizes where the central limit theorem fails. For example, the gradient noise is shown to be heavy-tailed through empirical evidence in \cite{revisit2}. In case the gradient noise is far from Gaussian distribution, the optimization path can be modeled as a delayed Lévy-driven OU process. More specifically, the optimization path can be characterized by the following SDDE with a non-Gaussian Lévy motion, 
    \begin{equation}
        \mathrm{d}x(t)  = - {v} x(t-\tau) \mathrm{d}t  + \mathrm{d}L\acute{e}(t),
    \end{equation}
    where $L\acute{e}(t)$ is a scalar Lévy motion and $L\acute{e}(0) = 0$. More specifically, a scalar Lévy motion is a stochastic process with stationary and independent increments. That is, for any $s, t$ with $0 \leq s < t$, the distribution of $L\acute{e}(t) - L\acute{e}(s)$ only depends on $t - s$, and for any $0 \leq t_0 < t_1 < \cdots < t_n$, $L\acute{e}({t_i} )- L\acute{e}(t_{i-1}), i = 1, \cdots, n$, are independent. 
\end{remark}}

\subsection{The Convergence Speed Given Gradient Noise}\label{sgd_grad_noise}
In this subsection, we derive a relationship between the convergence time and the gradient noise. To gain more insight on the impact of the gradient noise on the convergence time, we investigate the stochastic processes 
\begin{equation}\label{OU_delay_free}
  \mathrm{d}x_i(t)  = - v_i x_i(t) \mathrm{d}t  + \sigma \mathrm{d}B(t), 
\end{equation}
for $i = 1, 2, ..., N$ in this subsection.
In this way, the convergence time of SGD can be modeled as a first hitting time of (\ref{ou_mp}) to the $\delta$-neighborhood of the global optimum $\pmb x^*$, i.e.,
\begin{equation}\label{Hitting_NormINf}
    \tilde H = \inf \{ t> 0 : \left\|{\pmb x}(t) - \pmb x^*\right\|_{\infty} < \delta \},
\end{equation}
in which the $i$th element of $\pmb x^*$ is denoted by $x_i^*$.
Defining 
\begin{equation}\label{eq_def_Hi}
    H_i = \inf \{ t> 0 : |x_i(t) -  x_i^*|< \delta \},
\end{equation}
for $i = 1, 2, ..., N$. {Note that the definition of the first hitting time is presented in Subsection \ref{subc}.}
To further demonstrate the impact of the gradient noise on the convergence time of SGD, in this subsection, we first introduce the following proposition on the first hitting time according to \cite{FirstPass}.

\begin{proposition}
The expectation of the first hitting time of an OU process (\ref{OU_delay_free}) is given by 
\begin{equation}\label{exp_hittingtime}
\mathbb E\{ H_i\} = \frac{1}{v_i}\log \frac{x_{i,0}}{\delta},   
\end{equation}
in which $x_{i,0} = x_i(0)$.
\end{proposition}
The probability distributions of the first hitting times $H_i$ are however complex. To shed more light on the probability characteristics of $H$, we shall next show that the first hitting time $T_i$ is approximately Gaussian distributed. 
{
\begin{proposition}\label{proFirstHitting}
The first hitting time (\ref{Hitting_NormINf}) of the OU process (\ref{OU_delay_free}) can be approximated by a Gaussian distributed random variable
\begin{equation}\label{approximation}
    H_i \sim\mathcal N\left(\frac{\log \alpha_i}{v_i}, \frac{\sigma^2}{2 v_i^3\delta^2}\left(1- \frac{1}{\alpha_i^2}\right)\right).
\end{equation}
\end{proposition}
\begin{proof}
See Appendix \ref{appendix:1}.
\end{proof}}
Next, let $H_{\max} = \max_i\{H_i\}$ be the maximum of the first hitting times. In this context, a careful reader may notice that the $H_{\max}$ is not necessarily identical to  $H$, while the difference between $H_{\max}$ and $\tilde H$ will be detailed in Appendix \ref{appendix:2}


Further, $H_i$s are identically distributed {random variables} when $\pmb V=v\pmb I$ and $\hat {\pmb x}(0) = c\pmb 1$, in which $v>0$, $\pmb I$ denotes the identity matrix, and $\pmb 1$ denotes the vector of ones. In this case, ${H_{\max}}$ is given by the expectation of maximum of i.i.d. Gaussian variables. Let \begin{equation}
\Theta(N) = \mathbb E\{\Xi_N\}- \sqrt{\log N},
\end{equation} in which $\Xi_N = \max\{Z_1, Z_2, ..., Z_N\}$ and $Z_i$s are  independent and identically distributed (i.i.d.) standard Gaussian distributed {random variables}. That is,  $Z_i \sim \mathcal N(0, 1)$. More specifically, $\Xi_N$ is the maximum of $N$ i.i.d. standard  Gaussian random variables. In addition, $\Theta(N)$ is on the order of $\mathcal O(1)$, where $\mathcal O(\cdot)$ is the big $O$ notation. In particular, $\Theta(1)=0$. In this way, it holds that 
\begin{equation}\label{noise_effect}
\begin{aligned}
       \mathbb E \{H_{\text{max}}\} = \frac{1}{v}\log \alpha + \frac{\sigma\sqrt{1- \frac{1}{\alpha^2}}}{\sqrt{2v^3}\delta}\left(\sqrt{\log N} +\Theta(N)\right),
\end{aligned}
\end{equation}
in which $N$ is the dimension of the objective variable $\pmb x$.

\begin{remark}
When  $N = 1 $, it has been shown by Lemma \ref{lemma_averagespeed} that the gradient noise does not affect the expectation of the first hitting time in the SGD algorithm.    In particular, Eq. (\ref{noise_effect}) also suggests the same conclusion. Further, when $N \ge 2$, the presence of the gradient noise {introduces} an additional term on the average first hitting time, which {is proportional} to the standard variance of the gradient noise and $\sqrt{\log N}$. 
\end{remark}

Next, in order to make the problem analytically tractable and {obtain} insightful results, we consider the impact of gradient noise with { the objective function} $f(\pmb x) =\frac{1}{2} \pmb x^T\pmb I\pmb x$.
Thus, it holds that $\breve {\pmb x}[j+1] =\breve {\pmb x}[j]-\eta \breve {\pmb x}[j]   + \pmb w.$     
Note that $\eta$ is the learning rate and $\pmb w$ is isotropic Gaussian noise, i.e., $\pmb w\sim \mathcal N (0, \breve \sigma^2\pmb I)$. {To simplify the notation,} we define $\pmb r[j] = \breve {\pmb x}[j] -  \pmb x^*$ and $r = \left\|\pmb r \right\|_{2}$. 
Due to the {symmetry} of the normal distribution, the projection of $\pmb w$ onto the direction of $\pmb r$, denoted by $\pmb w_{\parallel}$, is also Gaussian distributed, i.e., $\left\|\pmb w_{\parallel}\right\|_{2} \sim \mathcal N (0, \breve\sigma^2)$. 
In this context, the component perpendicular to the $\pmb r$ direction is given by $\pmb w_{\perp} = \pmb w - \pmb w_{\parallel}$. In this case, we have $\langle\pmb w_{\perp}, \pmb r\rangle = 0 $, in which $\langle \pmb a, \pmb b \rangle $ denotes the inner product of {the two vectors} $\pmb a$ and $\pmb b$. In addition, we find that $\mathbb{E}\{\left\|\pmb w_{\perp}\right\|_{2}^2\} = (N-1)\breve \sigma^2$. 
Thus, it holds that
\begin{equation}\label{sde}
     r[j+1] =  r[j]- \eta r[j] + \left(\sqrt{r^2[j] + \left\|\pmb w_{\perp}\right\|_2^2 }- r[j]\right)  +   \pmb w_{\parallel},
\end{equation}
in which $w_{\parallel} = \left\|\pmb w_{\parallel}\right\|_2$.
For $r \gg\breve \sigma$, we have $\sqrt{r^2 + \left\|\pmb w_{\perp}\right\|_2^2 }- r = \frac{\left\|\pmb w_{\perp}\right\|_2^2 }{2r} + \mathcal O(1/r^2)$ by Taylor approximation $\sqrt{1+s} =   1+ \frac{1}{2}s + \mathcal O(s^2)$ near $s= 0$. For a given $r$, let the backward speed be defined as $z = \frac{\left\|\pmb w_{\perp}\right\|_2^2 }{2r}$ In this case, it can be seen that the presence of the noise $\pmb w_{\perp}$ hinders the process of the gradient descent, the speed of which is given by $z$.  More specifically, $z$ is Chi-square distributed with mean given by $\mathbb E\left\{z\right\} = \frac{(N-1)\breve \sigma^2}{2r}$, and variance given by $\text{Var} \left\{z \right\} = \frac{(N-1)\breve\sigma^4}{r^2}$. When $N = 2$, the backward speed $z$ follows an exponential distribution. 
Therefore, the continuous-time approximation of Eq. (\ref{sde}) can be formulated as follows:
\begin{equation}\label{sde2}
    \mathrm{d} r = -\left[\eta r - \frac{(N-1)\breve\sigma^2}{2r} \right]\mathrm{d}t + \sqrt{\sigma^2 + \frac{(N-1)\breve \sigma^4}{r^2}} \mathrm{d} B(t).
\end{equation}
The first hitting time of the above stochastic process (\ref{sde2}) is defined as $ H_d = \inf \{ t> 0 : |r(t) - r^*| < \delta, r(0) = r_0 \},$
in which $r^* = 0$.\footnote{Note that the $H_d$ is defined as the first time that the 2-norm of $\pmb r$ is smaller than a pre-given threshold. This is different from the definition presented in Eq. (\ref{Hitting_NormINf}), which is based on the infinity norm. } By Lemma \ref{lemma_averagespeed}, we obtain the following corollary.
\begin{corollary}
The expectation of the first hitting time $H_d$ is given by $\mathbb E\{H_d\} =\mathcal G(r_0)- {\mathcal G}(\delta)$, in which 
    $\mathcal G(r) = \frac{1}{2\eta}\log\left[{\eta r^2 - \frac{(N-1)\breve \sigma^2}{2}}\right]$.
\end{corollary}
\begin{proof}
     By integrating the primitive function ${\left[\eta r - \frac{(N-1)\breve \sigma^2}{2r}\right]^{-1}} $ with respect to $r$, and applying Lemma \ref{lemma_averagespeed}, we arrive at the corollary.
\end{proof}

In this way, one can notice that the presence of the gradient noise does not affect the expectation of the first hitting time of the gradient descent algorithm when $N=1$.
Further, the above analysis also indicates that as the value of $r$ decreases, the adverse impact of gradient noise on convergence becomes increasingly severe. This also {implies} that as the training progresses, increasing the number of workers involved in training or enlarging the mini-batch size,{ both of which help reduce gradient noise,} could accelerate the convergence of SGD.  Note that a high-resolution quantization of the gradients with large gradient noise can lead to a lot of network bandwidth usage, thereby being inefficient due to the waste of communication resources.

\subsection{The Convergence Speed Given Gradient Staleness} 
\label{sub:relationship_between_gradient_staleness_and_the_convergence_rate}
In this subsection, we present an analysis of {SGD with stale gradients}  through an SDDE-based approach. For the sake of discussion, we first {examine} the case of fixed staleness, before extending our discussion to the case in which $\tau$ is a random variable. Here, let the gradient staleness in SDDE be denoted by a positive constant $\tau$. 
To understand the behavior of the solution to {a} deterministic differential equation, we substitute the sampling function $\varrho(t) = Ae^{{\lambda} t}$, in which ${\lambda}$ is a complex number, into the delay differential equation (\ref{determine_function}), we have the following characteristics equation:
\begin{equation}\label{char_eq}
    {\lambda} e^{{\lambda} \tau} = - {v}.
\end{equation}
The function $h({\lambda}) =  {\lambda} e^{{\lambda} \tau} $ is monotonically decreasing on the interval $\left( -\infty, -\frac{1}{\tau} \right)$ and monotonically increasing on the interval $\left(-\frac{1}{\tau},\infty \right)$. It has a minimum of $\inf_{{\lambda}} h({\lambda}) = -\frac{1}{e\tau}$ at $h({\lambda}) = \frac{1}{\tau}$. Moreover, $\lim\limits_{{\lambda}\to -\infty}h({\lambda}) = 0$.

Hence, Eq. (\ref{char_eq}) has a single root when ${v}<0$. In this case, the SGD algorithm with a stale gradient overdamps and does not converge to the global minimum. In the following discussion, we shall first focus on the case $v>0$ before extending our discussion to the case $v\ge 0$. In particular, the characteristic equation has two negative real roots when $0<{v} \tau< \frac{1}{e}$. In this case, the SGD algorithm with a stale gradient decays monotonically to the global minimum of the quadratic objective function. Moreover, the characteristics equation has a pair of complex conjugate roots with negative real parts when 
    $\frac{1}{e}<{v}  \tau<\frac{\pi}{2}$.
In this case, the SGD algorithm with stale gradient undergoes damped oscillations. When ${v} \tau > \frac{\pi}{2}$, the real parts of the roots of Eq. (\ref{char_eq}) becomes positive. Therefore, the  SGD algorithm undergoes diverging oscillations.  Note that $\eta$ is the step size of the SGD algorithm. According to the above discussion, we have the following theorem for the discrete-time SGD.
\begin{theorem}\label{th1}
{ When $v> 0$}, the two roots of the characteristics function Eq. (\ref{char_eq}) are given by
\begin{equation}\label{roots}
    \lambda_k = \frac{W_k{(-{v} \tau)}}{\tau},
\end{equation}
for $k = 0,1$, in which $W_k$ is the Lambert W function. 
\end{theorem}
Further, we have the following corollary.
\begin{corollary}\label{two_roots}
Given the learning rate $\eta$ and the step staleness $\breve \tau$, the characteristic roots of the continuous-time approximation of the {ASGD} algorithm, $\breve {x}[j+1] = \breve {x}[j] - \eta v \breve{x}\left[{j - \breve \tau}\right]$, are given by
 $\lambda_k = \frac{W_k{(-{v}\eta\breve \tau)}}{\eta\breve \tau}$ for $k = 0 , 1$.
\end{corollary}
For the solvable case, the stochastic gradient descent algorithm with stale gradient converges to the only fixed point when $\breve\tau< \frac{ \pi}{2\eta{v}}$, and without oscillations when   $\breve\tau< \frac{1}{e\eta{v}}$.

With a fixed step size, $\lambda_k$  are real and negative in which $|\lambda_0|\le  |\lambda_1|$ when $\breve \tau < \frac{1}{e\eta{v}}$. In addition, it can be noticed that $|\lambda_0|$ increases as $\breve \tau$ increases for 
\begin{equation}\label{small}
    \breve \tau \in \left(0, \frac{1}{e\eta{v}}\right).
\end{equation}
This indicates that the presence of a small degree of gradient staleness accelerates the stochastic gradient descent algorithm, which aligns with the analysis in \cite{slow_and_stale}. However, the SGD algorithm starts to oscillate when $\breve \tau \in (\frac{1}{e\eta{v}}, \frac{\pi}{2\eta{v}})$. In addition, the real part of the roots {becomes} positive when $\breve \tau> \frac{\pi}{2\eta{v}}$. It indicates that, the SGD algorithm diverges with a large degree of gradient staleness.

\begin{remark}
To reveal the characteristic roots of (\ref{determine_function}) in case $\tau$ is a random variable, we will present the characteristic roots with uniformly distributed gradient staleness in the forthcoming discussion. Consider {a} Gaussian distributed gradient staleness, i.e.,  $\tau \sim \mathcal N(\vartheta, \varsigma^2)$. Numerically, we find that it holds that $\lambda_k \approx W_k(-v\vartheta)/\vartheta$ for $k = 0,1$ when $\varsigma \ll \vartheta$.\footnote{More specifically, we have the following observation of Gaussian distributed gradient staleness through numerical results. When $\tau \sim \mathcal N(\vartheta, \varsigma^2)$ and $\varsigma \ll \vartheta$, {ASGD} nearly has the same convergence rate as when the gradient staleness is a constant, i.e., $\tau  = \vartheta$. With a sufficiently large $\varsigma$, {ASGD} may diverge despite a small $\vartheta$. Analytical results for the staleness of Gaussian distribution gradients are  noted as future work. }
\end{remark}

Given the gradient staleness, a sensible choice of the learning rate with $f(x) = \frac{1}{2}{v} x^2$ is  
\begin{equation}\label{eq_bsq_lr}
    \eta = \frac{1}{ e v\breve \tau}.
\end{equation}
For $f(\pmb x) = \frac{1}{2} (\pmb x - \pmb x^*)^T \pmb V (\pmb x - \pmb x^*)$ in which $\pmb V$ is a diagonal matrix with its diagonal elements given by $0\le v_1\le  v_2\le \cdots \le  v_N$, a sensible choice of the learning rate is given by
\begin{equation}\label{necc}
    \eta \in \left(  \frac{1}{e v_N\breve \tau},  \frac{\pi}{2 v_N\breve \tau}\right).
\end{equation}
The reason is that if $\eta\ge \frac{\pi}{2 v_N\breve \tau}$, the SGD algorithm diverges. If $\eta < \frac{1}{e v_N\breve \tau}$, the SGD algorithm converges too slow. 

\begin{remark}
Assuming an $L$-Lipschitz continuous objective function $f(\pmb x)$, the step size must be chosen as $\eta< \frac{\pi}{2L\breve \tau}$ to avoid unbounded oscillations. A more conservative choice $\eta< \frac{1}{e L\breve \tau}$ ensures that no oscillation occurs. 
In this case, a small degree of gradient staleness accelerates the SGD algorithm, which aligns with the conclusion in \cite{berkeleyDuchi}, and \cite{add_usefulornot} in which quantization error can be exploited to accelerate the model convergence. Moreover, an appropriate staleness accelerates the convergence of SGD in the direction of eigenvectors with small eigenvalues and hinders others with large eigenvalues.
The above discussion aligns with the observation made by \cite{slow_and_stale} that {a stale gradient} may contribute to a faster convergence rate but a higher error floor. In some existing literature \cite{stale1,stale-aware}, the stale-aware step size is inversely proportional to stale, which also aligns with our analysis. However,  step size must be carefully chosen with further consideration of the 2-norm of the Hessian matrix of the objective function, to avoid unbounded oscillations. 
\end{remark}

{In the above discussion, it has been revealed that the convergence rate of the SGD is determined by $v$ and gradient staleness in SDDE $\tau$. In the following, we shall present the stationary distribution of ASGD with gradient staleness by \cite{small}.
\begin{proposition}
When $0 <v\tau < \frac \pi 2$, the stationary distribution of the SDDE $\mathrm{d}x(t)  = - {v} x(t-\tau) \mathrm{d}t  + \sigma \mathrm{d}B(t)$ is Gaussian with zero mean and variance
\begin{equation}\label{add_station}
    \sigma_{\pi} = \frac{\sigma^2}{2v}\left( \frac{1 + \sin(v\tau)}{\cos(v \tau)}\right).
\end{equation}
for $0 <v\tau < \frac \pi 2$.
\end{proposition}

Using a Taylor expansion for $\tau \ll 1$, we obtain $\sigma_{\pi} = \frac{\sigma^2}{2v}\left(1+ v\tau \right) + \mathcal O(\tau^2)$. Here, the result Eq. (\ref{add_station}) aligns with our earlier analysis in Eq. (\ref{roots}), where ASGD diverges when $v\tau > \frac \pi 2$.
Additionally, ASGD is found to have a larger variance in its stationary distribution compared to synchronous SGD without gradient staleness. This implies that ASGD has a greater steady-state error, defined as the difference between the desired value and the actual value of a system once the response has reached the steady state. To improve stability and convergence in practical applications, one could reduce the level of asynchrony or increase the mini-batch size, particularly when the algorithm is approaching a local minimum.
}

{
In the following discussion, we are interested in the characteristic roots of the SDDE when $v \le 0$. This analysis is particularly relevant for understanding the behavior of ASGD near saddle points. When $v < 0$, there are two characteristic roots, $\tilde \lambda_0 = \frac{W_0(-v\tau)}{\tau}$ and  $\tilde \lambda_1 = \frac{W_1(-v\tau)}{\tau}$. In this case, $\tilde \lambda_0 $ is real and positive, while $\text{Re}\left\{\tilde \lambda_0\right\}< 0$. In the following, we primarily focus on the dominant characteristic root $\tilde \lambda_0$ which has a positive real part. For brevity, we denote $\tilde \lambda_0$ by $\tilde \lambda$.
\begin{proposition}
If $v < 0$, the real and positive characteristic root of Eq. (\ref{char_eq}) is given by 
\begin{equation}
   \tilde \lambda(\tau) = \frac{W_0(-v\tau)}{\tau}.
\end{equation}
\end{proposition}
Recall that $W_0(\cdot)$ is the Lambert W function. Note that to differentiate from the case when $v > 0$, we denote the characteristic root as $\tilde \lambda$ specifically when $v < 0$. 
Next, we have the following corollary.
\begin{corollary}
It holds that
\begin{equation}\label{fk_add1}
   \lim_{\tau \to 0} \tilde \lambda(\tau) = -v.
\end{equation}
\end{corollary}
\begin{proof}
Since $\frac{\rm d}{\rm dt}W_0(0) = 1$, we obtain Eq. (\ref{fk_add1}) by L'Hôpital's rule.
\end{proof}
Further, we have the following corollary revealing the monotonicity of $\tilde \lambda(\tau)$.
\begin{corollary}
When \(v < 0\), the function \(\tilde{\lambda} = \frac{W_0(-v\tau)}{\tau}\) is a decreasing function of \(\tau\) on the interval $[0, \infty)$.
\end{corollary}
\begin{proof}
The derivative of $\tilde{\lambda}$ with respect to $\tau$ is given by 
\[
\frac{\rm d}{\rm d \tau}\tilde{\lambda}(\tau) = \frac{W_0(-v\tau) \left(-\frac{v}{W_0(-v\tau) + 1} - 1\right)}{\tau^2}.
\]
For \(v < 0\), \(W_0(-v\tau) > 0\) since \(-v\tau > 0\). The term \(-\frac{v}{W_0(-v\tau) + 1} - 1\) is negative, leading to \(\tilde{\lambda}'(\tau) < 0\). Thus, \(\tilde{\lambda}\) is decreasing with respect to $\tau$.
\end{proof}
Therefore, we have the following proposition, which demonstrates that gradient staleness slows down the escape from a saddle point.
\begin{proposition}
It holds that $\tilde \lambda < -v$ when $v< 0$.
\end{proposition}
Next, we have the following approximations on the characteristic root.
\begin{proposition}\label{small_tau}
    If $\tau \ll 1$, the characteristic root can be approximated as $\tilde \lambda = -v(1 + v \tau) + \mathcal O(\tau^2)$.
\end{proposition}
\begin{proof}
Since $W_0(s) = s- \frac{s^2}{2} + \mathcal O(s^3)$ for $s\ll 1$ by Taylor expansion, we arrive at the conclusion.
\end{proof}

As stated in Assumption 1, the objective function is assumed to be twice differentiable, which motivates further insightful discussions highlighting the impact of the gradient staleness on the behavior of the SGD algorithm based on the positive definiteness of the Hessian matrix. As presented above, a small degree of gradient staleness accelerates convergence, but a large degree of gradient staleness causes divergence near critical points where the eigenvalues of the Hessian matrices are positive. For the sake of completeness, we next highlight the impact of gradient staleness on the behavior of SGD near critical points, where some eigenvalues of the Hessian matrix can be negative, to provide further insights. In particular, the presence of gradient staleness results in a slower escape from the critical point with only negative eigenvalues. This occurs because the effective gradient magnitude is scaled by a factor of \( (1 + v \tau) \) when \( \tau \ll 1 \) and \( v < 0 \). Near a saddle point, the escape speed depends on both the magnitude of gradient staleness and the Hessian matrix of the objective function. Gradient staleness can either accelerate or decelerate the escape from the saddle points.
Consider a two-dimensional objective function, where the eigenvalues of the Hessian matrix are \( v_1 > 0 \) and \( v_2 < 0 \) and  \( \tau \ll 1 \). When \( |v_2| < |v_1| \), the presence of gradient staleness increases the escape speed from the saddle point. In contrast, when \( |v_2| > |v_1| \), the equation
\(
v_1^2(1 + v_1 \tau)^2 + v_2^2(1 + v_2 \tau)^2 = v_1^2 + v_2^2
\)
has a single non-zero positive root, at which point the escape velocity from the saddle point matches that of the non-stale case.

Next, we have the following discussion on the characteristic root when $\tau\gg 1$.
\begin{proposition}\label{big_tau}
For $\tau \gg 1$, it holds that $\tilde \lambda = \frac{\log(-v\tau) - \log\log(-v\tau)}{\tau} + \mathcal O\left(\frac{\log\log(-v\tau)}{\tau\log {-v \tau}}\right)$.
\end{proposition}



Lastly, we consider the dynamics of ASGD along the eigenvector where the eigenvalue is zero, i.e.,  $v = 0$. By Eq. (\ref{ou_p_tmp}), the dynamics of ASGD evolve according to
\begin{equation}\label{ou_proc}
        \mathrm{d}x(t)  = - q\mathrm{d}t  + \sigma \mathrm{d}B(t),
\end{equation}
which is a Brownian motion with drift. In this case, it can be seen that the gradient staleness has no impact on the training process by Eq. (\ref{ou_proc}).  Next, without loss of generality, we focus on the case $q \le 0$. In this case,  for a Brownian motion starting at $x_0$, let us next present the first hitting time distribution of Eq. (\ref{ou_proc}) to $x_0 + \delta$, where $\delta > 0$ is a constant. More specifically, the first hitting time is defined as $H_0 = \inf\{t\ge 0: x(t)\ge x_0 + \delta, x(0) = x_0, q< 0 \}$, where $x(t)$ evolves according to Eq. (\ref{ou_proc}). Note that $H_0$ is the first hitting time to $[x_0 + \delta, \infty)$, or equivalently, the first exit time from $(-\infty, x_0 + \delta)$. In this case, the probability distribution of $H_0$ is given by $ p_{H_0}(t) = \frac{\delta}{\sqrt{2\pi t^3}}\exp{-\frac{(\delta - qt )^2}{2t}},~ t>0.$}

{
Next, the objective variable evolves according to 
\begin{equation}\label{small_stale_acc_sdde}
    \mathrm{d}x(t)  = -v (1+v\tau) x(t) \mathrm{d}t  + {(1+v\tau)} \sigma\mathrm{d}B(t),
\end{equation} for $\tau \ll 1$ by \cite{small}. In this way, we have the following remark for the connection between ASGD with stale gradients and SGD with momentum.

\begin{remark}
From Eq. (\ref{small}), we observe that a small degree of staleness can accelerate the convergence.  Specifically,  the drift coefficient is $-v(1+v\tau)$ in Eq. (\ref{small_stale_acc_sdde}) with a gradient staleness $\tau$, while the drift coefficient is $-v$ in Eq. (\ref{OU_delay_free}) without gradient staleness. This indicates that the introduction of gradient staleness increases the learning rate by a factor of $(1+\tau v)$. A careful reader may also observe a connection between small gradient staleness and the stochastic gradient descent method with momentum. Both approaches utilize information from previous gradients,  which contributes to improved convergence rates. In particular, the formula of SGD with momentum is given by 
\begin{equation}
\breve {\pmb x}[j+1] = \breve {\pmb x}[j] - \eta \left( \nabla f (\breve {\pmb x}[j]) \right) + \beta (\breve {\pmb x}[j] - \breve {\pmb x}[j-1]).
\end{equation}
By \cite[Section 7.3.1]{bishop} and \cite[Chapter 8.3.2]{goodfellow}, if the momentum algorithm always observes a gradient $\pmb g$, then 
\begin{equation}
    \Delta {\pmb x}[j]   =  - \frac{\eta}{1- \beta}\pmb g.
\end{equation}
where $\Delta {\pmb x}[j]  = \breve {\pmb x}[j+1] - \breve {\pmb x}[j]$.
More specifically, the result of the momentum term is to increase the effective learning rate from $\eta$ to $\eta/(1-\beta)$ \cite{bishop}. In this context, a small degree of gradient staleness and SGD with momentum both increase the convergence rate by leveraging the information of previous gradients. This effect is to increase the effective learning rate parameter.
 Further, we can express both gradient update methods in a unified form as follows:
\begin{equation}
 \Delta {\pmb x}[j] = \sum_{l=0}^{j- 1}\eta_l \left( \nabla f (\breve {\pmb x}[j- 1]) \right)  +  \beta\Delta{\pmb x}[j - l].
\end{equation}
For SGD with momentum, we set $\eta_0 = \eta$ and $\beta_1 = \beta$,  while all other parameters are zero.  For ASGD with a constant staleness $\breve \tau$, we have $\eta_{\breve \tau} = \eta$,  with all other parameters remaining zero.
\end{remark}
}

In the following, we present the performance analysis when the gradient staleness is a uniform random variable, which can arise in event-triggered SGD. More specifically, consider the delay differential equation given below:
\begin{equation}\label{}
    {\rm d} x(t) = -v x(t- \tau(t)) {\rm d}t,
\end{equation}
in which $\tau(t_1)$ and $\tau(t_2)$ are assumed to be identically and independently distributed random variables when $t_1 \ne t_2$. Further, it holds that $\tau(t) \sim U(0, \zeta)$ and $\mathbb E\{\tau \}= \frac{1}{2}\zeta$. In the following discussion, we study the first moment of $\langle x(t) \rangle$, in which brackets $\langle\cdot \rangle$ denotes an average over the realizations of the random process $\{\tau(t)\}$. In this way, we have the following proposition.
\begin{proposition}\label{prop_diff_delay}
When the gradient staleness is uniformly distributed, the average $P(t) = \langle x(t)\rangle$ obeys the second order differential equation given by
\begin{equation}\label{diff_eq_uni_delay}
    \frac{{\rm d}^2P(t)}{{\rm d}t^2} = - \frac{v}{\zeta}\left(P(t) - P(t- \zeta) \right).
\end{equation}
\begin{proof}
Let $Q(t) = \int_{t-\zeta}^t \langle x(s) \rangle {\rm d}s$.
In this way, it holds that 
    $\frac{{\rm d} P(t)}{\rm dt} = -v\frac{Q(t)}{\zeta}$ and $\frac{{\rm d} Q(t)}{{\rm d} t}= P(t) - P(t-\zeta).$
By this means, we arrive at Eq. (\ref{diff_eq_uni_delay}).
\end{proof}
\end{proposition}
By Proposition \ref{prop_diff_delay}, we have the subsequent corollary illustrating the characteristic equation associated with the differential equation given by Eq. (\ref{diff_eq_uni_delay}). Without loss of generality, we assume $v= 1$ for the sake of discussion in the rest of this subsection.
\begin{corollary}
When $v=1$, the differential equation (\ref{diff_eq_uni_delay}) has the following characteristic equation: 
\begin{equation}\label{chara_function}
    \zeta \lambda^2  =e^{-\zeta\lambda} - 1.
\end{equation}
\end{corollary}
In the following, we present a brief discussion on the solution to the characteristics function Eq. (\ref{chara_function}).

\begin{corollary}\label{coro_solution_for_delay}
The characteristics equation Eq. (\ref{chara_function}) has two distinct negative real roots for $\zeta \in (0, \frac{\varpi ^2}{e^{-\varpi }-1})$, has repeated real roots $\lambda^* = \frac{e^{-\varpi}-1}{\varpi}$ for $\zeta = \frac{\varpi ^2}{e^{-\varpi }-1}$, has two complex roots for $\zeta \in (\frac{\varpi ^2}{e^{-\varpi }-1}, \infty)$ in which $\varpi = -W_0\left(-\frac2{e^2}\right)-2$.\footnote{The numerical value of  $\varpi$ is  $\varpi \approx - 1.5936$. Also, it follows that $\frac{\varpi ^2}{e^{-\varpi }-1} \approx 0.6476$ and $\lambda^* \approx -2.4608$.} The real part of the complex roots is negative for $\zeta \in \left(\frac{\varpi ^2}{e^{-\varpi}-1}, \frac{\pi^2}{2}\right)$, and is positive for $\zeta \in \left(\frac{\pi^2}{2}, \infty\right)$.\footnote{Note that there is always a trivial solution $\lambda = 0$ for Eq. (\ref{chara_function}).}
\end{corollary}
\begin{proof}
    See Appendix A.
\end{proof}



\section{Parameter Optimization for Distributed SGD}\label{sec_pra_dis}

In this section, we will further present run-time and staleness analysis for distributed SGD, building upon the presented theoretical performance analysis.  Following this, we will reveal the protocol design criteria for distributed SGD.

\subsection{Parameter Optimization for {ASGD}}\label{sub_toadd}
In this subsection, we present the probability distribution of the update interval, also referred to as the run-time\cite{slow_and_stale}, and step staleness in $B$-ASGD to bridge the {ASGD} algorithm and the SDDE-based continuous approximation. Recall that in $B$-ASGD, the central server collects the first $B$ gradients while the other clients do not halt the computation of their gradient. For the sake of discussion, we assume infinite communication bandwidth in this subsection, in which the communication delays are negligible compared to the computational delay. In this case, the gradient delay or the continuous-time gradient staleness equals the computational delay.
The computational delay of each client is assumed to follow a general distribution, denoted by $T_\kappa$, where $\kappa$ represents the client index.
Let $N_p(t)$ be the aggregated arrival process of gradients to the central server in $B$-ASGD. 
Let $Y\sim \text{Exp}(\nu)$ denote a random variable following an exponential distribution with a mean $1/\nu$, i.e., a rate $\nu$. 

{
Let $N_\kappa(t)$ be the $\kappa$-th gradient arrival process, which is defined as the number of gradients collected by the central server until time $t$ from the $\kappa$-th worker. Clearly, $N_\kappa(t)$ is a step process, where the increments $N_\kappa(t) - N_\kappa(s)$, take on only non-negative integral values. In addition, a step process $N(t)$ is a Poisson process when
$ \text{Pr}\{N(t) - N(s) = k \} = \exp\{-(\Lambda(t) - \Lambda(s))\}\frac{(\Lambda(t) - \Lambda(s))^k}{k!},$
where $\Lambda(t)$ is the rate function of the Poisson process, $k \in \mathbb N_0$, and $\mathbb N_0$ is the set of non-negative integers.
Let the aggregated process of $K$ step processes be denoted by $N_p^K(t) = \sum_{\kappa = 1}^{K}N_\kappa(t)$. Here, $N_p^K(t)$ has the same meaning as $N_p(t)$ defined above. Here, we use the superscript $K$ when it is necessary to emphasize that $N_p^K(t)$ is a superposition of $K$ step processes.
Next, let us introduce the following notations. Let $p_\kappa(k; t,s) = \Pr\{N_\kappa(t)- N_\kappa(s)\} = k$ for $s < t$ and $k \in \mathbb N_0$. Let $\Lambda_K(t,s) = \sum_{\kappa =1 }^{K}p_\kappa(1; t, s)$ and $\Omega_K(t, s) = \sum_{\kappa =1}^{K}(1 - p_\kappa(0; t,s) - p_\kappa(1; t,s))$. Next, we make the following assumption on the step processes \cite[Theorem 30.3]{add_bokk}.
\begin{assumption}\label{add_ass3}
It is assumed that the gradient arrival processes satisfy  (1) $\lim_{K\to \infty}\Lambda_K(t,s) = \Lambda(t) - \Lambda(s)$, (2) $\lim_{K\to\infty} \Omega_K(t, 0) = 0$, and (3) $\lim_{K\to \infty} \max_{1\le \kappa \le K}(1- p_\kappa(0; t, 0)) = 0$.
\end{assumption}

{In the following, we provide a justification for Assumption 3, highlighting its practical relevance in the context of large-scale distributed learning. First, the Poisson approximation holds under the condition that the number of participating sub-processes tends to infinity, {which aligns with the fact that a large number of workers are involved in the training process in large-scale distributed learning.} Second, it is required that no two events (i.e., gradients from different clients) arrive simultaneously, which is reasonable as the probability of such simultaneous arrivals is practically zero.  {Third, for $N_p(t)$ with a rate $\Lambda_K$, the arrival rate of each sub-process should remain of lower order than $\mathcal O (\Lambda_K)$ as $K \to \infty$}, which implies that the gradient contribution from any single client does not dominate the overall process.
{Furthermore, in the case of a bandwidth-limited communication channel with a fixed service rate,} to ensure the stability of the queues and prevent their divergence, the aggregate gradient arrival rate is a pre-given constant or function with respect to time, as presented in Assumption 3, and will be detailed in Subsection IV-D.}
\begin{lemma}\label{lemma_sup}
Under Assumption \ref{add_ass3}, the aggregated gradient arrival process in pure ASGD $N_p^K(t)$ converges to a Poisson process with a rate function $\Lambda_K(t)$ as $K\to \infty$ in case $N_{\kappa}(t)$ are mutually independent.
\end{lemma}
Furthermore, \cite[Theorem 2]{add_poisson} presents an error analysis for approximating \( N_p^K(t) \) by a Poisson process with finite \( K \). This result demonstrates that, even if the individual step processes do not converge to a limiting process, the finite sum can nonetheless be well-approximated by a Poisson process.

}
In this way, we have the following corollary, which does not rely on exponentially distributed computation time \cite{op, op2, op3}.
\begin{corollary}\label{up_interval}
The update interval in $B$-ASGD is given by
\begin{equation}\label{update_interval}
    I = \sum_{i = 1}^{B} I^{(i)},
\end{equation}
in which $I^{(i)} \sim \text{Exp}((K -i +1)\nu)$ and  $\mathbb E\{I\} = \frac{1}{v} \sum_{i = K-B +1}^{K} \frac{1}{i}$ in case $\mathbb E\{T_\kappa\} = 1/\nu$ for $\kappa \in \{1, 2, ..., K\}$.
Also, it holds that 
\begin{equation}\label{expectation_interval}
    \mathbb E\{I\}  \approx \frac{1}{v}\left[\log\frac{K}{K-B} -  \frac{B}{2K(K-B)}  \right] .
\end{equation}
\end{corollary}
\begin{proof}
By the Poisson approximation of the aggregated arrival process, the update interval between the arrival of the  $(i-1)$th and $i$th gradients is exponentially distributed with rate $(K- i + 1)\nu$ for $i \in \{1, ..., B\}$. By this means, we get Eq. (\ref{update_interval}). Next, we have Eq. (\ref{expectation_interval}) {by the Euler–Maclaurin formula}, {where the approximation error is on the order of $\mathcal O(1/{K^2})$}.
\end{proof}

{
In synchronous SGD, the parameter server updates the global model only after receiving all the gradients from the workers. Therefore, the time between two consecutive updates in synchronous SGD is given by
\begin{equation}\label{SYN_slotLength}
    T_{s}^{\text{SYN}} =\max\{\hat \tau_1, ...,\hat \tau_K\},
\end{equation} 
where $\hat \tau_\kappa$ represents the continuous-time gradient delay of Worker $\kappa$.

\begin{remark}
To illustrate the straggler problem in synchronous SGD, we compare it with how asynchronous SGD addresses this issue. We focus on a scenario with $K=2$ workers as an example, each experiencing an exponentially computational delay. Additionally, the communication delays are assumed to be negligible. Note that in this case the continuous-time gradient delay is identical to the computational delay.  Let $\hat \tau_1 \sim \text{Exp}(\nu_1)$ and $\hat \tau_2 \sim \text{Exp}(\nu_2)$, where it is assumed that $\nu_1 \le \nu_2$ without loss of generality. In this context, Work 1 is the straggler. The rate of the aggregated gradient arrival process in ASGD is $\nu_p = \nu_1+ \nu_2$. By Eq. (\ref{SYN_slotLength}), we have $T_{s}^{\text{SYN}} = \max\{\hat \tau_1, \hat \tau_2\}$. Let $\nu_p^{\text{SYN}} = \frac{1}{\mathbb{E} \{ T_{s}^{\text{SYN}}\}}$. 
Therefore, $\nu_p^{\text{SYN}} =\frac{2}{\frac{1}{\nu_1} + \frac1{\nu_2} - \frac{1}{\nu_1 + \nu_2}}$. It can be verified that $\nu_p > \nu_p^{\text{SYN}}$. The reason for this result is that the central server has to wait for the slowest worker before it can update the global model in synchronous distributed learning. Further, it holds that $\nu_p^{\text{SYN}} = \frac{2}{\frac 1{v_1} +  \frac{v_1}{v_2{(v_1 + v_2)}}}$. Hereby, we see that $\nu_p^{\text{SYN}}  \approx 2\nu_1$, when $\nu_2 \gg \nu_1$ as the contribution from the term $\frac{v_1}{v_2{(v_1 + v_2)}}$ becomes negligible compared to that of $\frac{1}{v_1}$. This illustrates that synchronous SGD is slowed down by the straggler worker.
\end{remark}}

Next, we present the expectation of the step staleness in $B$-ASGD. Note that in $B$-ASGD, the step staleness equals $BJ$ rather than $J$ if the global model takes $J$ iterations during the period that the  client fetches the global model and finishes the computation. This is due to the fact in each iteration, the global model is updated based on all the $B$ gradients.  
\begin{theorem}\label{th_stepstale}
The expectation of the step staleness in $B$-ASGD without gradient dropout is given by\footnote{When $B =1$, this theorem reduces to the pure {ASGD} case, which was presented by Eq. (18) in \cite{uot}.}
\begin{equation}\label{step_stale}
    \mathbb E \{ \breve \tau\} = K-B.
\end{equation}
\begin{proof}
Let $S^\kappa[J] = S^\kappa[J-1]$ if the Client $\kappa$'s contributes to one of the $B$ gradients in the $J$th iteration of global model and $S^\kappa[J] = S^\kappa[J-1] + B$ otherwise in which $S^\kappa[0] = 0$. In this way, $S^\kappa[J]$ equals the sum of step gradients of all the gradients received by the central server from Client $\kappa$ plus the step staleness of the current computing gradient. 
Let $S[J]= \sum_{\kappa = 1}^{K} S^\kappa[J]$.  In this way, $S[J] = (K-B)BJ$ since $B$ gradients arrive at the central server and the other $K-B$ gradients' step staleness is increased by $B$. The total gradients received by the central server is $BJ$ after $J$ times global model updates.  Since $\mathbb E\{\breve \tau\} = \lim_{J\to \infty}\frac{S[J]}{BJ}$, we arrive at  Eq. (\ref{step_stale}).
\end{proof}
\end{theorem}

It is observed that $B$-ASGD, the reduction of the step staleness is achieved by decreasing the update frequency. Next, we present a brief discussion on the difference between $B$-ASGD and $B$-batch ASGD. Similar to Theorem \ref{up_interval},  the update interval follows an Erlang distribution, denoted by $\text{Erlang}(B, \nu)$, with mean $B/\nu$ in $B$-batch ASGD. Further, following the discussion in Theorem \ref{th_stepstale}, it holds that $\mathbb E \{ \breve \tau\} = K-1$ in $B$-batch ASGD.

To further reduce the step staleness, {{a}} client could choose to stop the computation of the current gradient according to a continuous-time or discrete-time staleness threshold. 
Next, we first present the following discussion on $B$-ASGD with a continuous-time staleness threshold. In the rest of this subsection, we assume each client's computational delay follows an i.i.d. distribution, denoted by $T$.
Let $F_T(\cdot)$ be the cumulative distribution function of $T$. In this context, let $p =  F_T(\hat \tau_{th})$ be the probability that a computation of gradient is not discarded due to timeout. Further, $\bar p = 1- p$. Let $\tilde T$ be the truncated computation time, in which $F_{\tilde T}(s) = \frac{F_T(s)}{p}$ if $s <\hat \tau_{th}$ and $F_{\tilde T}(s) = 0$  for $s>\hat \tau_{th}$. In this case, $\mathbb E\{\tilde T \} = \frac 1p \int_{s= 0}^{\hat \tau_{th}} s{\rm d}{F_T(s)}$.
Next, we have the following proposition on the inter-update time of the $B$-ASGD with a continuous-time staleness threshold.  
\begin{proposition}
In $B$-ASGD with a continuous time staleness threshold, the interval $\Upsilon$ between consecutive updates of gradients from a specific client is given by $\Upsilon  = M\hat \tau_{th}  +\tilde T$, in which $M$ is a geometric random variable with probability mass function $\text{Pr}\{M = m\} = p\bar{p}^{m}$ for $m \in \mathbb N$. 
\end{proposition}

By the above proposition, we have the following corollary.
\begin{corollary}
The expectation of $\Upsilon$ is given by 
\begin{equation}\label{con_tuning}
    \mathbb E\{\Upsilon\} = \frac{\bar p}{p}\hat \tau_{th}+\frac{1}{p} \int_{s= 0}^{\hat \tau_{th}} s{\rm d}{F_T(s)}.
\end{equation}
\begin{proof}
By the expectation of the sum of random variables, we arrive at Eq. (\ref{con_tuning}).
\end{proof}
\end{corollary}
Therefore, the rate of the aggregated gradient arrival process is given by $\tilde \nu = \frac{K}{\mathbb E\{\Upsilon\}}$. By Theorem \ref{up_interval}, we have the following corollary.
\begin{corollary}
    The update intervals in $B$-ASGD with a continuous time staleness threshold is given by $I = \sum_{i = 1}^{B} \text{Exp}((K -i +1)\tilde \nu)$, in which $\tilde \nu = \frac{K p}{ (1-p)\hat \tau_{th}+\int_{s= 0}^{\hat \tau_{th}} s{\rm d}{F_T(s)}}$ 
    and $\mathbb E\{I\} = (\sum_{i = K-B +1}^{K} \frac{1}{i})/\tilde \nu $.
\end{corollary}

Next, we turn our attention to $B$-ASGD with a discrete-time gradient staleness threshold $\breve \tau_{th}$. In particular, if the global model has been updated $\breve \tau_{th}$ times before a client finishes its computation of the current gradient, the client shall stop the computation of the current gradient and start the computation of a new one. In this case, the step staleness of each gradient is no greater than $B\breve \tau_{th}$. With memoryless computation time, we have the following proposition. 

\begin{proposition}\label{prop_tuning}
Assuming i.i.d. exponential computation times across each client, the step staleness with a discrete-time staleness threshold $\breve \tau_{th}$ follows a truncated geometric distribution, i.e., 
\begin{equation}
    \text{Pr}\{\breve \tau  = mB\} = \frac{q\bar q^m}{1- \bar q^{\breve\tau_{th} +1}} \text{ for } \breve \tau = 0, B,..., B\breve \tau_{th},
\end{equation}
in which $q = B/K$. Further,  
\begin{equation}
\mathbb E \{\breve \tau \} =\frac{B}{1- \bar q^{\breve\tau_{th} +1}}\left( \frac{\bar q - \bar q ^{\breve\tau_{th} +1} }{q} - \breve\tau_{th}\bar{q}^{\breve\tau_{th} +1}\right).
\end{equation}
\end{proposition}
By Proposition \ref{prop_tuning}, we have two observations. First, $\lim_{\breve\tau_{th} \to \infty}\mathbb E \{\breve \tau \} = K-B$, which fits with the results presented in Theorem \ref{th_stepstale}. Next, by tuning the computation of the gradient in terms of gradient staleness, one could also reduce the gradient staleness.

\subsection{Parameter Optimization for Pure {ASGD}}
In this subsection, we focus on pure {ASGD}, or $1$-ASGD, for deeper insights, in which $B=1$ and $\hat \tau_{th} = \infty$. Here we relax the assumption of identically distributed computation time across each client.
Let $T_\kappa$ be the continuous-time random variable corresponding to the computational delay of the $\kappa$th worker. Further, it is assumed that $0 < \mathbb E\{T_\kappa\} < \infty$. Let $\nu_\kappa = \frac{1}{\mathbb E\{T_\kappa\}}$ and $\nu_p = \sum_{\kappa = 1}^{K}  \nu_\kappa$.
The sequence of the computational delay from Worker $\kappa$ is denoted by $\left(T_\kappa^{(i)}\right)_{i\ge 1}$. The update time of each worker can therefore be denoted by $J^{(n)}_{\kappa} = \sum_{i = 1}^{n} T_{\kappa}^{(i)}$. Let us define the following stochastic process    $N_{\kappa}(t) = \sum_{n =1}^{\infty}\mathds{1}\left\{J^{(n)}_{\kappa} \le t \right\}$
representing the number of updates from Worker $\kappa$ by time $t$. Furthermore, the rate of $N_{\kappa}(t)$ is  $\nu_\kappa$.
In this context, the number of arrived gradients at the parameter server is given by the superposition of $N_{\kappa}(t)$ for $\kappa = 1,2, ..., K$. More specifically, we have $N_p(t) = \sum_{\kappa =1}^{K}N_{\kappa}(t).$ From Theorem \ref{th_stepstale}, we directly have the following corollary.
\begin{corollary}
In pure {ASGD}, it holds that $\mathbb E\{\breve \tau\}= K-1$.
\end{corollary}

First, we aim to present the distribution of the gradient staleness and the optimal choice for the number of workers in pure ASGD. Let $Y \sim \text{Pois}(\nu)$ be the random variable distributed according to a Poisson distribution with rate $\nu$. In this context, we have the following proposition, which demonstrates that the discrete-time gradient staleness follows a mixed Poisson distribution \cite{mixGau}. 
\begin{proposition}\label{prop_mixPois}
In pure ASGD, the discrete-time gradient staleness of Client $\kappa$'s gradient follows a mixed Poisson distribution, i.e., 
\begin{equation}\label{poisson_approx}
    \breve \tau_{\kappa} \sim \text{Pois}((\nu_p - \nu_\kappa) T_{\kappa}).
\end{equation}
In this way, it holds that  
$\mathbb E\{\breve \tau_k\} = (\nu_p- \nu_\kappa)/ \nu_\kappa,$
and  $\text{Var}\{\breve \tau_k\} = (\nu_p- \nu_\kappa) / \nu_\kappa + (\nu_p- \nu_\kappa)  ^2 \sigma_\kappa^2,$ in which $\text{Var}\{T_\kappa\} = \sigma_\kappa^2$.

\end{proposition}
\begin{proof}
From the perspective of Worker $\kappa$, the superposition of the arrival gradients from all other workers can be modeled as a Poisson process with rate $\nu_p - \nu_\kappa$. By this means, we arrive at Eq. (\ref{poisson_approx}). 
\end{proof}

Next, the discrete-time staleness is a mixed Poisson distribution with the rate parameter distributed according to an exponential distribution. Therefore, we have the following corollary.
\begin{corollary}\label{coro_EXP}
Given $T_{\kappa} \sim \text{Exp}\left(\nu\right)$, the gradient staleness in {ASGD} with {an} ideal communication channel follows a geometric distribution, i.e., 
$   \breve \tau_{\kappa} \sim \text{Geo}\left(\frac{1}{K-1}\right),$
in which $K$ is the total number of workers. Further, we have $\mathbb E \{ \breve \tau_k\} = {K-1}$ and $\text{Var}\{ \breve \tau_k\} = K(K-1)$.
\end{corollary} 

This corollary aligns with Theorem \ref{th_stepstale} when $B = 1$. Further, the impact of the increase in the number of clients on the convergence of SGD is twofold. On the one hand, there is a higher volume of gradient data arriving within a given time duration. On the other hand, the increase in the number of clients can lead to a greater gradient staleness.

\subsection{Optimal Worker Numbers in {ASGD} Through an Ideal Data Link or Network}
In this subsection, our objective is to  determine the optimal choice of the number of workers in {ASGD} along the eigenvector with an eigenvalue $v > 0$, assuming an ideal data link. Here, an ideal data link refers to one that is error-free and zero-delay. In pure {ASGD}, the expectation of the step staleness is $\mathbb E \{\breve \tau \}= K-1$.  Based on this fact, we present the following proposition.

\begin{proposition}\label{prop_opt_number_ideal}
When $K \gg 1$, for pure {ASGD} through an ideal channel with a fixed learning rate, the optimal number of workers, defined as the number of workers under with the dominant root of SGD dynamics achieves its minimum, denoted by  $K^*$, is given by 
 \begin{equation}
     K^* = 1 + \frac{1}{v\eta e}.
 \end{equation}
\end{proposition}
\begin{proof}
See Appendix \ref{app:add}.
\end{proof}

\begin{remark}
In this discussion, we aim to analyze the effect of the number of workers on the convergence rate of pure {ASGD} when using a worker number-aware step size. The expected value of the discrete-time gradient staleness is given by $\mathbb E \{\breve \tau\} = K - 1$, in which $K$ {denotes} the number of workers. In this case, { the worker number-aware step size is given by $\eta = \eta_0/K$,} where the constant $\eta_0 > 0 $. The dominant characteristic root is $\lambda_0 = W_0(-\eta_0 v)/\eta_0$. Further, the update frequency is $\nu = {K\nu_0}$, where $\nu_0$ is the update frequency of one single worker. Consequently, by the dominant pole approximation, the SGD dynamics evolve as $\mathcal O(e^{ {W_0(-\eta_0 v)\nu_0t}})$ by the dominant pole approximation. {This analysis indicates that increasing $K$ does not effectively improve performance in {ASGD} when using a worker number-aware step size.}

Also, given the learning rate $\eta$ and the total number of workers $K$, it can be seen that when $K \ge  1+ \frac{1}{v\eta e}$ the performance of pure {ASGD} deteriorates due to the large step staleness. In this case, one could improve the performance by setting an appropriate $B$. A sensible choice of $B$ can be $B^* = K - \frac{1}{v\eta e}$ by Theorem \ref{th_stepstale}.
\end{remark}

\subsection{Optimal Worker Number in Pure {ASGD} Through a Bandwidth Constrained Channel Shared by Multiple Workers}\label{opt_num_afl}

In this subsection, we aim to determine the optimal number of workers in pure {ASGD} operating through a shared medium with limited bandwidth. In this case, the continuous-time gradient staleness is characterized by the sum of the computational delay and the communication delay.  In this scenario, an increase in $K$ not only results in more frequent updates but also leads to a rise in the overall gradient delay due to increased communication delay.

Specifically, we assume a constant service time for the communication channel, given that the arriving data packets are approximately uniform in size. Still, it  holds that $\mathbb E \{T_\kappa\} = \frac{1}{\nu}$, $\text{Var}\{T_\kappa\} = \sigma^2_\kappa$. The service rate of the communication channel is denoted by $\mu_{\text{Q}}$. In this way, the service time for each data packet is a deterministic time $T_{\text{Q}} = \frac{1}{\mu_Q}$ seconds. As per Lemma \ref{lemma_sup}, the arrival process of gradients to the communication channel can be approximated as a Poisson process with rate $\nu_p =K\nu $. Therefore, the communication channel can be modeled as a $M/D/1$ queue.
Note that $\nu_p$ is the rate of $N_p(t)$. In this way, the utility of the $M/D/1$ queue can be denoted by $\rho  = \frac{\nu_p}{\mu_Q}$ in which $\rho < 1$. Therefore, the average waiting time of the gradients in the queue is given by $W = \frac{1}{\mu_Q} \frac{\rho}{2(1-\rho)}.$

Further, the average communication delay, which is given by the sum of the average waiting time in the queue plus the average service time, is $\mathbb E\{D\} =  \frac{1}{\mu_Q} \left[\frac{\rho}{2(1-\rho)} + 1\right]$. In this way, the expectation of the continuous-time gradient staleness is given by the sum of the expectation of the computation and communication delay. More specifically, we have 
$\mathbb E\{\hat \tau\}   = \frac1{\nu} + \frac{1}{\mu_Q} \left[\frac{\rho}{2(1-\rho)} + 1\right].$
With limited bandwidth, the client can choose to start the computation of a new gradient right away after it finishes the computation of the previous one. It can also start the computation of the new gradient after it receives an ACK (acknowledgment) that the global model has been updated according to the gradient, which reduces the gradient staleness but also reduces the update frequency. In the latter situation, it still holds that $\mathbb E\{\breve \tau\}= K-1$. In the subsequent, we focus on the former situation and we have the following proposition.
\begin{proposition}
With $K$ workers in total, the expectation of step staleness in {ASGD} through a channel with a fixed service rate is given by\footnote{The arrival process of the gradients to the central server is the departure process of the $M/D/1$ queue. Rigorously, the departure process of the $M/D/1$ is not a Poisson process. But the intensity of the departure process equals the intensity of the arrival process to the queue as long as $\mu_Q > \nu$. }
\begin{equation}\label{aver_delayMD1}
    \mathbb E \{\breve \tau_\text{Q}(K)\} = (K -1)\left[1+ \frac{\nu}{2\mu_Q} \left(\frac{K\nu}{\mu_Q- K\nu} + 2\right)\right],
\end{equation}
in which $\mu_Q$ is the service rate of the communication channel. 
\end{proposition}
Therefore, starting a computation of a new gradient without waiting for an ACK increases the average step staleness with a bandwidth-limited channel. By Eq. (\ref{aver_delayMD1}), we present the following remark.

\begin{remark}
The impacts of $K$ on the gradient staleness $\mathbb E \{\breve \tau_\text{Q}(K)\}$ are twofold. First, the increase of $K$ lead to an increase of the communication delay, as indicated by the term $\frac{\nu}{2\mu_Q} \left(\frac{K\nu}{\mu_Q- K\nu} + 2\right)$ on the right-hand side of Eq. (\ref{aver_delayMD1}). Further, for a specific worker, as $K$ increases, more gradients from other workers arrive within a given period of time, leading to a linear increase with respect to $K$ in staleness, as indicated by the term $K$ on the right-hand side of Eq. (\ref{aver_delayMD1}). 
Note that $\mathbb E \{\breve \tau_\text{Q}(K)\}$ is monotonically increasing with respect to $K$. 
\end{remark}

Next, we have the following proposition on the optimal number of workers in pure {ASGD},
\begin{proposition}\label{prop7}
Given the learning rate $\eta$ and the queue capacity $\check K$, the optimal number of workers in pure {ASGD} can be determined as the smaller solution to the following quadratic equation:
\begin{equation}\label{eq_prop7}
  \left(1 + \frac{1}{2\check K}\right)x^2 - \left(1 + \check K + \frac{1}{e\eta v}\right)x + \frac{\check K}{e\eta v} = 0,
\end{equation}
in which the capacity of the queue is $\check K=\frac{\mu_Q}{\nu}$.
\end{proposition}
\begin{proof}
See Appendix B.
\end{proof}

In practice, the workers can reduce the {computation time, i.e., the amount of time a worker takes to compute the gradient,} by reducing the size of the minibatch, thus increasing the frequency of gradient updates. However, this approach can lead to significant bandwidth consumption and network congestion, consequently degrading algorithm performance. Moreover, a high-resolution quantization of the gradients with large gradient noise also leads to substantial bandwidth usage, thus being inefficient due to the waste of network bandwidth.

\subsection{Optimal Design for Event-Triggered Distributed Optimization}
In this subsection, we present a discussion on distributed learning with event-triggered communication, which is presented in the literature to effectively save the bandwidth \cite{gupta}.

Given the state variable at $t_0$, denoted by $\pmb x_0$, the determination of the triggering time can be formulated as the first exit time of a continuous-time stochastic process (\ref{con_adl}), i.e., $H_E = \inf\{t\ge t_0| \pmb x(t)\in  E(\pmb x_0),  \pmb x(t_0) = \pmb x_0 \},$
in which $ E(\pmb x) = \{\tilde{ \pmb x}| \left\|\nabla f(\tilde {\pmb x}) -  \nabla f(\pmb x)\right\| \ge \xi\}$. Following the discussion in Subsection \ref{sgd_grad_noise}, we assume that the triggering time   is a Gaussian random variable, i.e., $H_E \sim\mathcal N(\vartheta, \varsigma^2)$. In this case, gradient staleness also arises in event-triggered SGD.
Given $H_E = \zeta$, the conditional gradient staleness approximately follows a uniform distribution $U(0, \zeta)$. Therefore, the gradient staleness follows a compound distribution, i.e., $\tau \sim U(0, H_E)$. Further, it is assumed that $\vartheta \gg \varsigma$. By the law of total expectation, it holds that $\mathbb E \{\tau\} =   \frac{\vartheta}{2}$.
Further, the p.d.f. of $ \tau$ is given by $ p_{\tau}(s) = \int_s^{\infty} \frac{1}{\sqrt{2\pi}\varsigma s}e^{-\frac{1}{2}\left(\frac{s- \vartheta}{\varsigma}\right)^2} {\rm{d}}s\text{, for } s\ge 0$. In this case, approximately uniformly distributed gradient staleness arises distributed SGD in with ETC when $\vartheta \gg \varsigma$.  
Next, for the solvable case, i.e., with the quadratic objective function $f(\pmb x) = \frac 12 \pmb x^T \pmb V \pmb x$, the expectation of the triggering-time in event-triggered SGD is given by $\mathbb E\{ H_E\} = \xi\sqrt{\frac{1}{\pmb x_0^T \pmb V^4 \pmb x_0}}$ by assuming constant gradient $\pmb V\pmb x_0$ when $\xi\ll \|\pmb x_0\|_2$. Note that different from the previous discussion, SGD along different eigenvector directions can not be viewed as independent in event-triggered SGD.

In distributed learning with event-triggered communication, the step staleness is $\breve \tau = \hat \tau/T_s$ in which $T_s$ is the timeslot duration in event-triggered SGD. In addition, we have $\tau = {\eta \hat \tau}/{T_s}$. Next, we focus on the optimal number of workers in event-triggered SGD given the total bandwidth constraint, in which we limit our discussion to the one-dimensional solvable objective function in which $v = 1$.
 
\begin{remark}
Let $C$ be the total bandwidth allocated to the uplink channel in event-triggered SGD. More specifically, the workers can update up to $C$ gradients to the parameter server per second on average. We assume that the triggering threshold is carefully chosen such that the communication channel is fully utilized and the workers are assumed to be homogeneous. Recall that the expectation of the triggering-time is denoted by $\mathbb E \{H_E\} = \vartheta$. In this way, the gradient staleness in SDDE is approximately uniformly distributed, i.e., $\tau \sim U(0, \vartheta)$, in which $\vartheta = \frac{K\eta}{CT_s}$. Note that $\eta$ is the fixed learning rate and $T_s$ is the time slot length in distributed learning with event-triggered communication. Similar to the discussion in Subsection \ref{opt_num_afl}, the optimal number of workers is approximately the number of workers at which the gradient staleness minimizes the real part of the dominant root. In this way, we conclude that the optimal number of workers is given by $ K^* = \frac{C\varpi^2 T_s}{\eta(e^{-\varpi }-1)}$ 
given a total bandwidth constraint in distributed learning with event-triggered communication. Given a total bandwidth limit, too many workers can induce a large degree of gradient staleness and the performance of event-triggered SGD can therefore deteriorate.
\end{remark}

\begin{figure*}[t]
    \centering
    \subfigure[The roots of the characteristic functions.]{\includegraphics[width=3.4in]{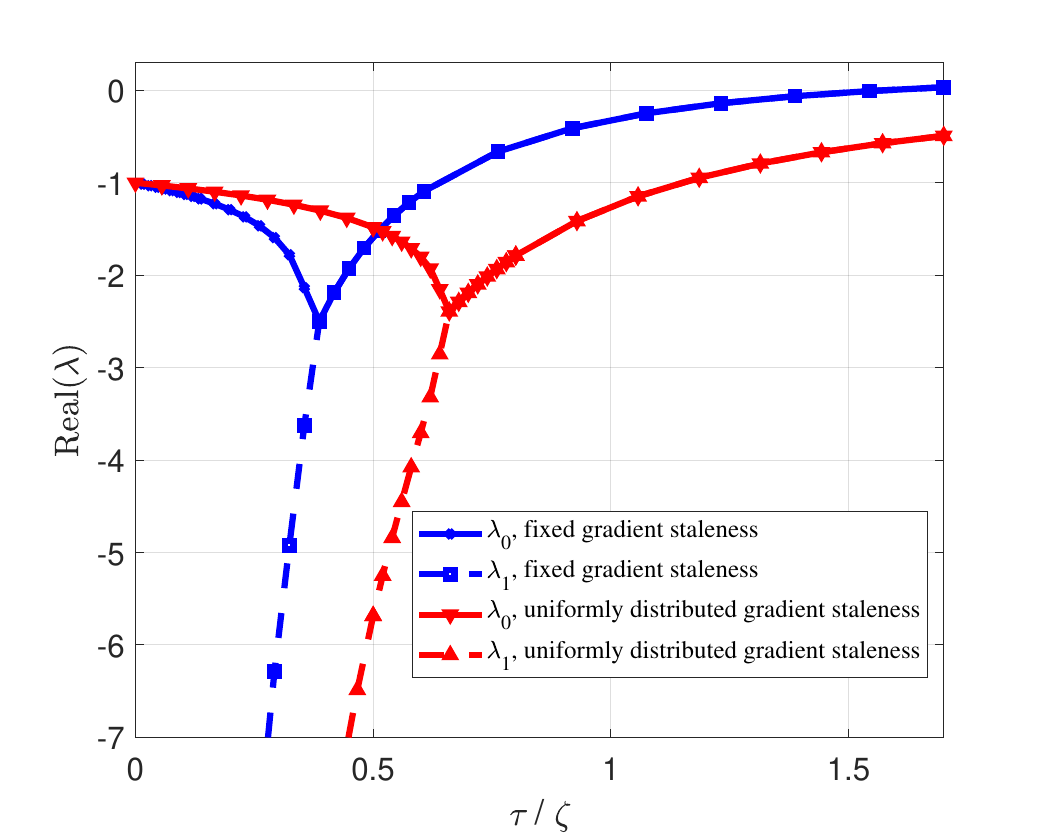}\label{fig_root} }
    \hfill
    \subfigure[The relationship between the convergence rate and the gradient staleness for the solvable cases.]{\includegraphics[width=3.4in]{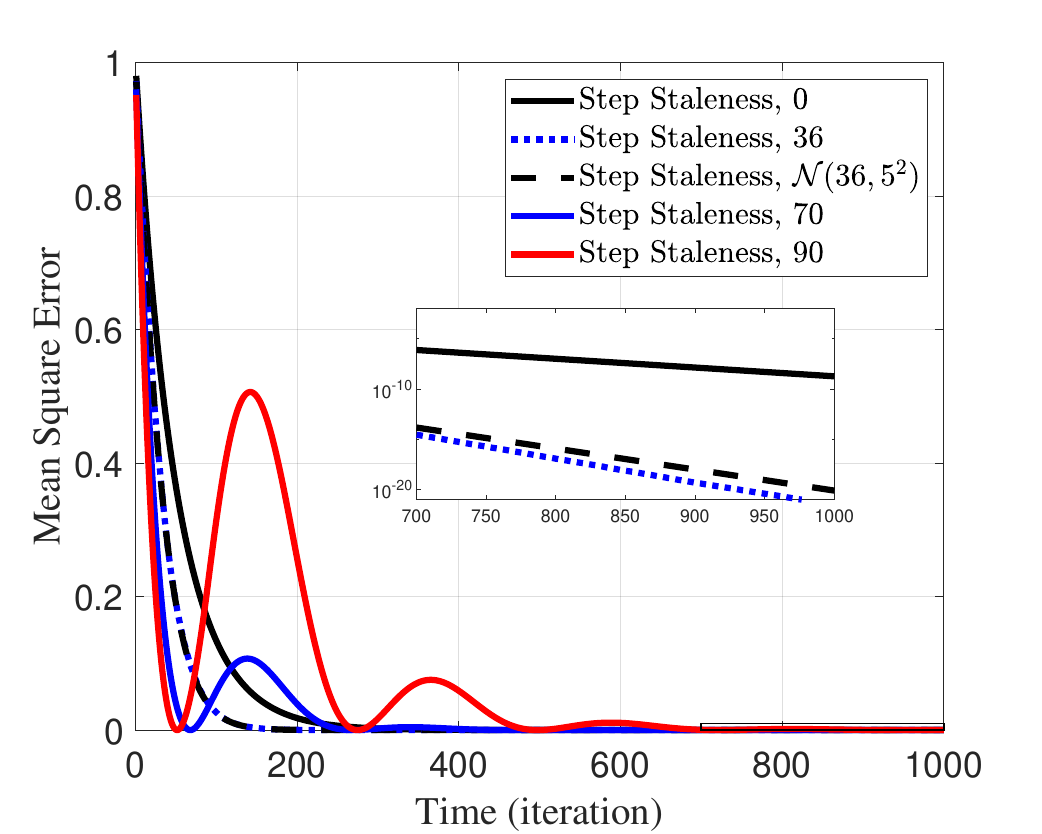}\label{fig_add}}
    \caption{The impact of $\tau$ on the convergence of the SGD algorithm.}
\end{figure*} 

\section{Numerical Results}

In this section, we shall present simulation and numerical results to verify the  theoretical analysis for distributed SGD with {stale gradients}. We will further demonstrate the effect of gradient noise, step size, and gradient staleness on {ASGD}.

To illustrate the relationship between the roots of the characteristic function and the gradient staleness,  Fig. \ref{fig_root} presents numerical results for the characteristic roots of Eq. (\ref{char_eq}) and Eq. (\ref{chara_function}). In particular, we consider a fixed gradient staleness, denoted by $\tau$. or a uniformly distributed gradient staleness, denoted by $\tau\sim \text{Uni}(0, \zeta)$, respectively. In both cases, the real part of the dominant root first decreases, then increases as $\tau$ or $\xi$ increases. With a fixed gradient staleness, the real part of the root achieves its minimum, i.e., $\text{Re}\{\lambda_0\} = -e$ when $\tau = 1/e$, while for uniformly distributed gradient staleness, the real part of the root achieves its minimum when $\zeta = 0.6476$.   The above results demonstrate that an SGD algorithm achieves performance gain with an appropriate gradient staleness. This may be counterintuitive that the gradient staleness does not always impede the convergence of the SGD method. However, the SGD process oscillates when $\tau> 1/e$ or $\zeta$, and diverges when $\tau > \frac{\pi}{2}$ or $\zeta > \frac{\pi^2}{2}$. This illustrates that the convergence property considerably deteriorates with a large gradient staleness.

To validate the aforementioned discussion on the roots of the characteristic function, Fig. \ref{fig_add} presents the effect of gradient staleness on the convergence rate of SGD with gradient staleness. To avoid the impact of gradient noise on SGD, we set $\sigma = 0$. Also, the learning rate is $\eta = 0.01$, and we adopt the same objective function as in Fig. \ref{fig_root}. We observe that, as the discrete-time gradient staleness increases from 0 to 36, an increased degree of gradient staleness leads to faster convergence. As the gradient staleness further increases, i.e., when $\eta \breve \tau> {1}/{e}$, SGD oscillates, which leads to {slower convergence}. Further, it can be seen that when $\breve \tau \sim N(36, 5^2)$, SGD nearly has the same {convergence rate} as when the discrete-time gradient staleness is a constant given by $\mathbb E\{\breve \tau\}$.

To reveal the optimal choice of the learning rate in multi-dimensional optimization, Fig. \ref{fig_lr} is presented. In particular, $f(\pmb x) = \frac{1}{2}\pmb x^T\pmb V\pmb x$ in which $\pmb V = \text{diag}(v,1)$. Notably, $\left\|\pmb V \right\|_2 = \max\{v, 1\}$. Given the discrete-time gradient staleness and $v$,  we numerically find the optimal (fixed) step size, i.e., the step-size that leads to the minimum loss after $1\times10^5$ iterations, by exhaustive search. For example, {we observe} that $\eta^* = \frac{1}{e\breve \tau}$ when $v =1$, which agrees with Eq. (\ref{eq_bsq_lr}). Notably, we observe that it always holds that $\eta^*\in\left[1/{(e\left\|\pmb V \right\|_2)}, {\pi}/{(2\left\|\pmb V \right\|_2)}\right)$, which is consistent with Eq. (\ref{necc}) in Subsection \ref{sub:relationship_between_gradient_staleness_and_the_convergence_rate}. This is due to the fact that, when $\eta> {\pi}/{(2\left\|\pmb V \right\|_2)}$, in the eigenvector direction which the largest eigenvalue, SGD diverges. {Additionally,} SGD converges slower along each eigenvector direction by further reducing the learning rate when  $\eta< {1}/{(e\left\|\pmb V \right\|_2)}$.

\begin{figure*}[t]
    \centering
    \subfigure[Number of iterations required for synchronous SGD to converge.]{\includegraphics[width=3.4in]{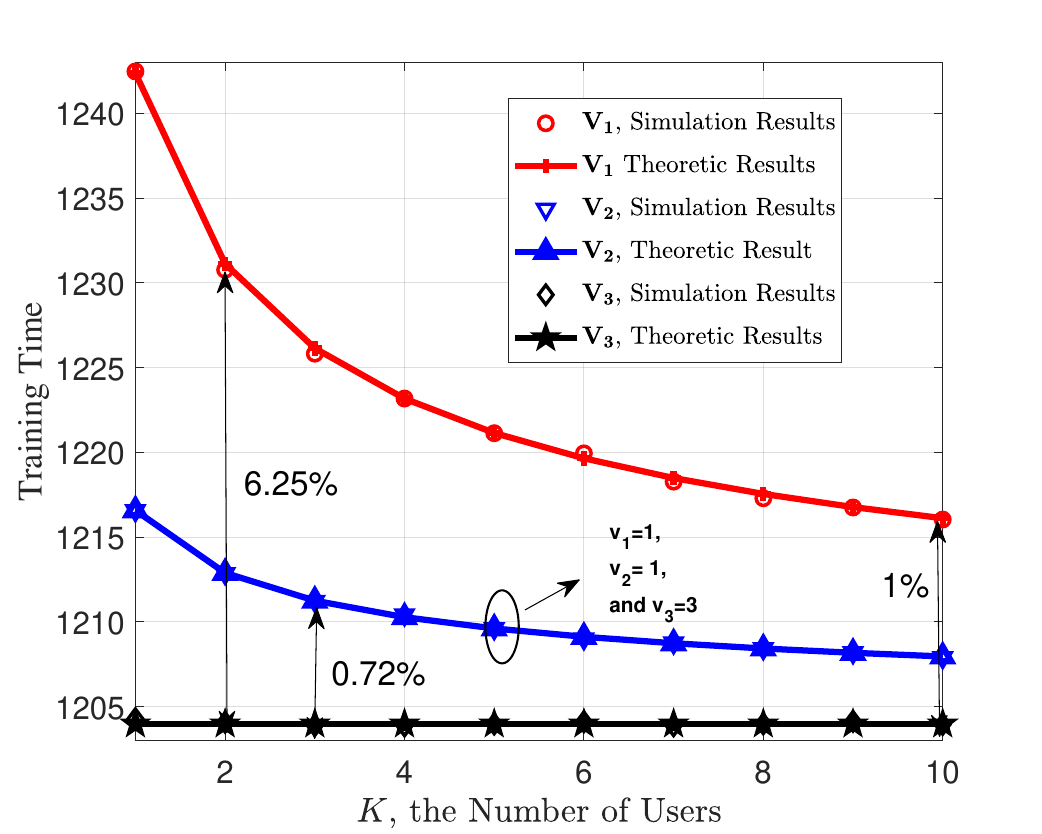}\label{fig1} } 
    \hfill
    \subfigure[The optimal learning rate for {ASGD}.]{\includegraphics[width=3.4in]{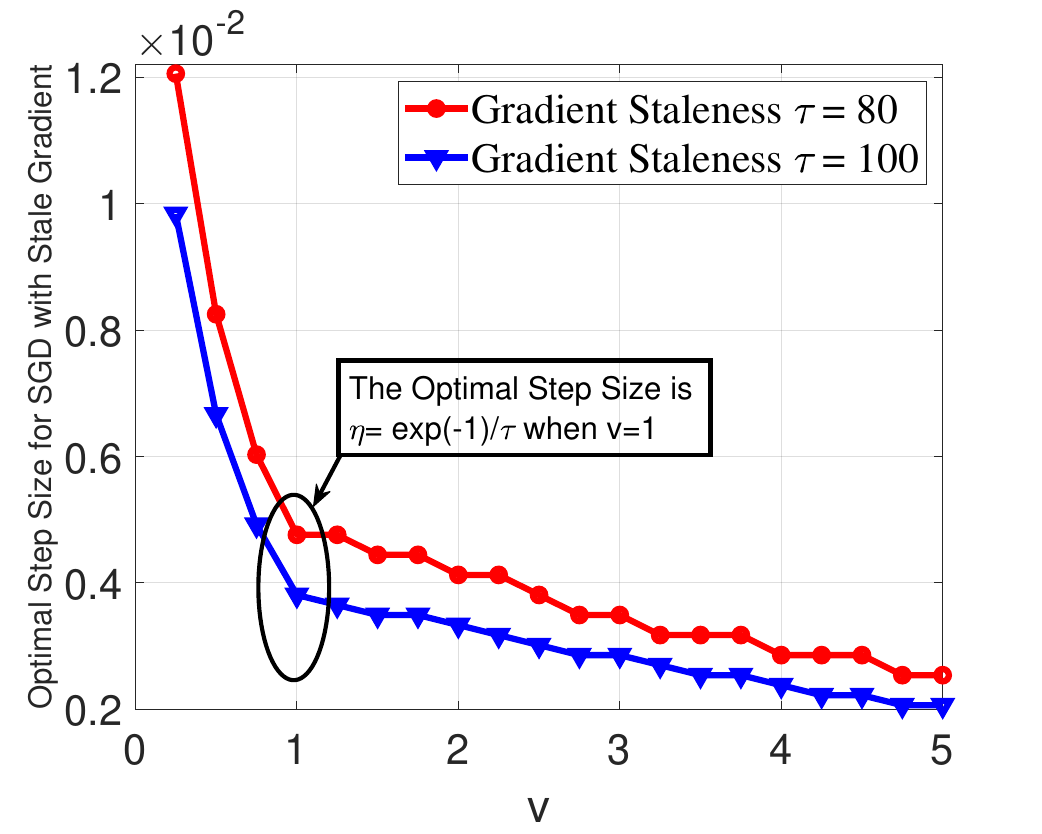}\label{fig_lr}}
    \caption{The impact of the gradient noise and learning rate on the convergence of the gradient descent algorithms.}
\end{figure*} 
In the subsequent, we verify our differential equation-based performance analysis on the effect of the gradient noise in Fig. \ref{fig1}.  To avoid the impact of gradient staleness on the convergence of SGD, we study $K$-synchronized SGD. The objective functions are $f_i(\pmb x) = \frac{1}{2}\pmb x^T \pmb V_i \pmb x$, in which $\pmb V_1 = \text{diag}(1,1,1)$, $\pmb V_2 = \text{diag}(1,1,3)$, and $\pmb V_3 = \text{diag}(1,2,3)$. Also, the variance of the additive Gaussian noise is inversely proportional to the square root of $K$, in which $K$ is the number of workers.  It is seen that in quadratic optimization, the number of iterations for an SGD algorithm is mainly determined by the smallest eigenvalue. Further, it can be observed that the average iterations for the SGD algorithm to meet the stopping criterion stays nearly unchanged despite {the increase} in the number of workers with the objective function $f_3(\pmb x)$. This is due to the fact that the average first hitting time of a one-dimensional OU-process remains unchanged with respect to the variance of the additive Gaussian noise as suggested by Lemma \ref{lemma_averagespeed}. Further, it is seen that the training time slightly decreases with the increase of $K$ with $\pmb V_1$, which fits well with Eq. (\ref{noise_effect}). This is due to the fact {that} the decrease in the noise variance leads to a decrease in the variance of the first hitting times along each eigenvector. In addition, we observe that, the differential-equation based approach presents an accurate performance analysis with a quadratic function. Also, it should be noted that the increase of $K$ reduces the variance of the training time.

 Next, we aim to verify the presented run-time and step staleness analysis for $B$-ASGD. Fig. \ref{fig_stale} presents the expectation of the step staleness versus $B$ in $B$-ASGD. In particular, we have $K=400$, in which $T\sim \mathcal N (10, 3^2)$ for $B$-ASGD, $T \sim  U(0,20)$  for $B$-ASGD with a continuous-time gradient staleness threshold $\hat \tau_{th} = 15$, and $T\sim \text{Exp}(1/10)$ for $B$-ASGD with a discrete-time gradient-staleness $\breve \tau =100$. It is observed that the numerical results fit well with the theoretical results, which verifies our discussions in Subsection \ref{sub_toadd}. Further, Fig. \ref{fig_time} presents the update interval of $B$-ASGD in which $K=400$, $B = 10$, and $T\sim \mathcal N (10, 3^2)$. It is seen the empirical results fit well with the theoretical results, which verifies Theorem \ref{up_interval} and the Poisson approximation of the aggregated gradient arrival process, even with non-exponentially distributed computation time.

In the following simulations, we investigate the effect of gradient noise and staleness and verify our performance analysis based on empirical results on the MNIST handwritten digit database. First, we present the effect of the number of workers on asynchronous learning with a fixed learning rate. In particular, the learning rate is $\eta = 0.01$. The computational delay of each worker follows an i.i.d. Gaussian distribution $\mathcal N (10,1^2)$ and the communication delay is assumed to be negligible. {In this case, we employ a convolutional neural network (CNN). More specifically, the  neural network consists of two convolutional layers followed by a fully connected layer. The first conventional layer  takes a single-channel input and applies 16 filters with a kernel size of 5. It has a stride of 1 and uses padding of 2 to keep the output size the same. After the convolution, it applies a ReLU activation function and then performs max pooling with a kernel size of 2. The second layer takes the output from the first layer and applies 32 filters with the same kernel size and parameters. It also uses ReLU activation and max pooling to further refine the features. The output from the second convolutional layer is flattened into a one-dimensional vector and fed into a fully connected layer that outputs 10 values, corresponding to digits 0 to 9. Additionally, the Hessian matrix is evaluated after 10 iterations from the start of the training process and the loss function selected for training is the cross-entropy loss. }

\begin{figure*}[t]
    \centering
    \subfigure[Step staleness in $B$-ASGD.]{\includegraphics[width=3.4in]{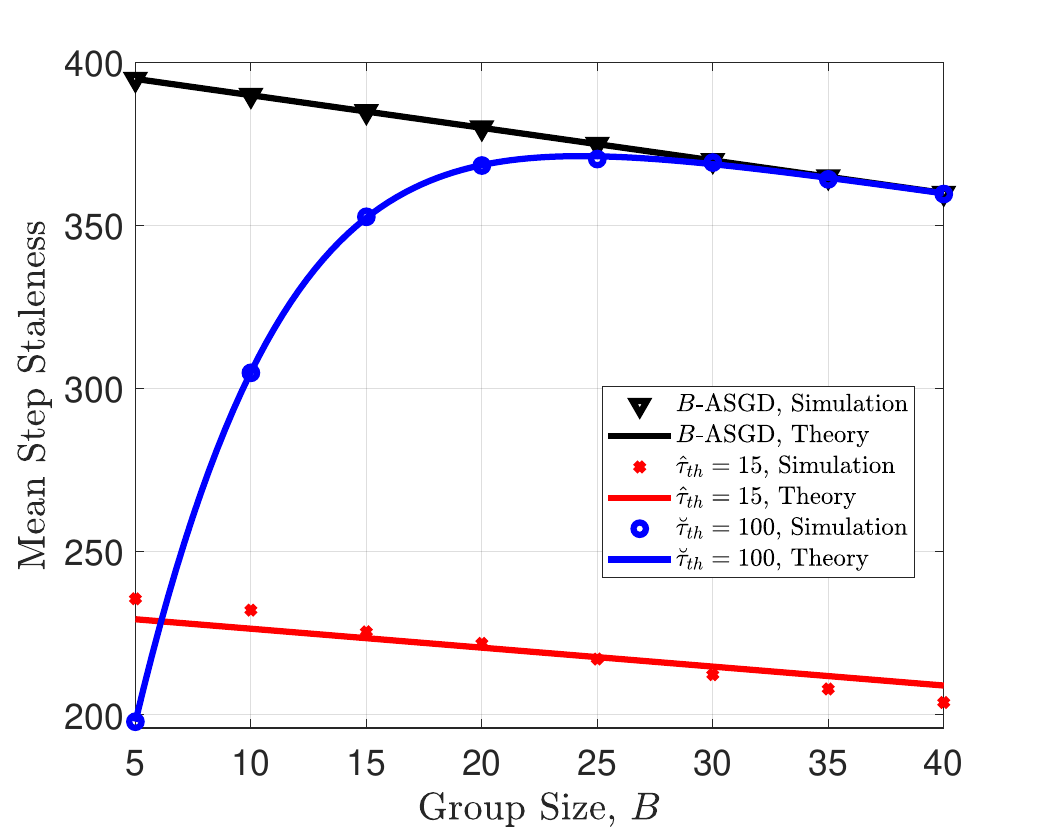}\label{fig_stale} }
    \hfill
    \subfigure[The probability distribution of the update interval in $B$-ASGD.]{\includegraphics[width=3.4in]{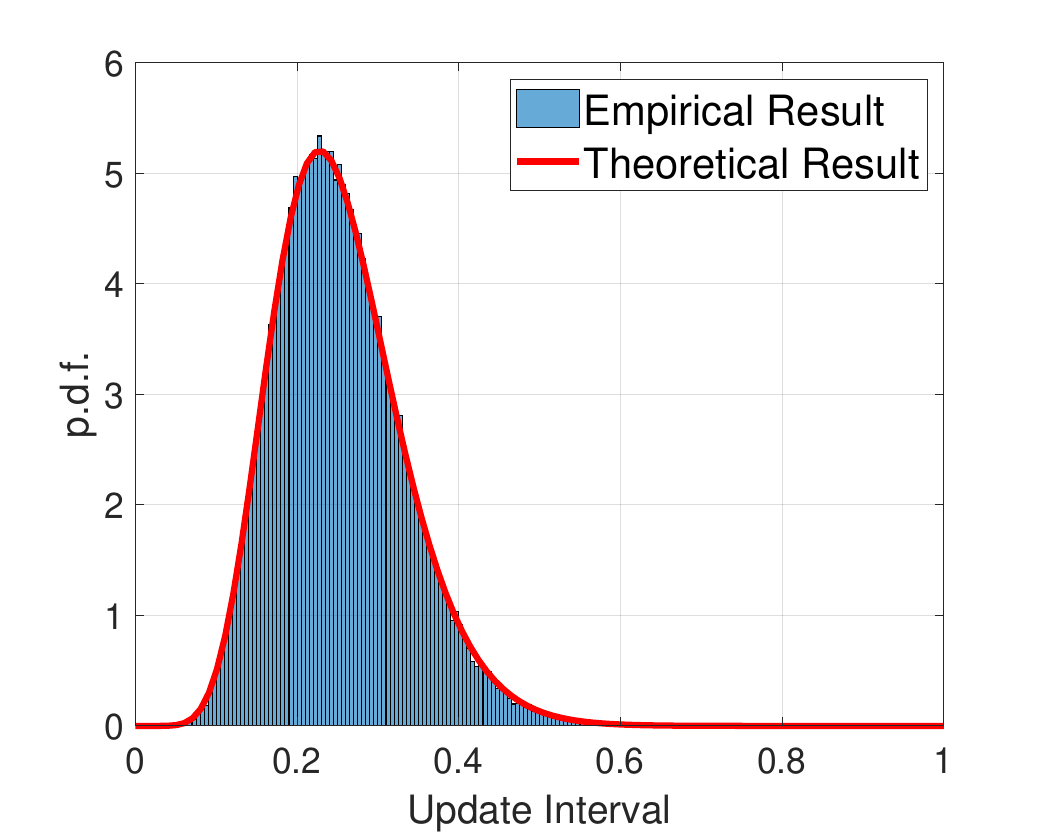}\label{fig_time} }
    \caption{Run-time and step staleness in $B$-ASGD.}
    \label{fig:runtime}
\end{figure*}

Fig. \ref{fig_mnist} presents the performance of asynchronous learning with respect to wall-clock time on the MNIST dataset with a CNN. First, similar to the solvable cases, the performance of asynchronous learning first increases with the increase of $K$ (or step staleness), then decreases. It is seen that when there are fewer workers, for example, $K= 5$, {ASGD} converges slowly.  This is due to the fact that with small $K$, the number of gradients collected by the global model is also limited. Further, when $ K = 20$, the {ASGD} algorithm achieves nearly the best performance. Another observation that can be made from Fig. \ref{fig_mnist} is that when $K \ge 20$, the performance of {ASGD} deteriorates, which diverges when  $K\ge 40$. The reason for this is that  gradient staleness becomes significant as $K$ increases, which leads to the divergence of {ASGD}. Note that the discrete-time gradient staleness is of significance in our analysis, which is defined as the number of global model updates between two consecutive gradient updates for the same worker. This is different from the continuous-time gradient staleness with respect to the wall-clock time. These results demonstrate that, too many workers in asynchronous learning can considerably deteriorate the efficiency of the algorithms.\footnote{In the simulation, we have also presented a rough estimation of the norm-2 of the Hessian matrix $\pmb H$ of the employed CNN, which indicates that $\left\|\pmb H \right\|_2$ falls in the interval $(3, 5)$ during the first several iterations. When $K = 20$, {ASGD} almost achieves the optimal performance. In this case, we have $\mathbb E \{\breve \tau\} \approx 20$. Hence, we have $\eta\mathbb E \{\breve \tau\} \left\|\pmb H \right\|_2 \in (0.6, 1.0)$. Note that, for the solvable cases, a necessary condition for {ASGD} to achieve optimal performance is $\eta\mathbb E \{\breve \tau\} \left\|\pmb H \right\|_2 \in ({1}/{e}, {\pi}/{2})$. Also, {ASGD} diverges when $\mathbb E \{\breve \tau\} \ge 40$, which  agrees with the argument in Eq. (\ref{necc}) for solvable cases in which {ASGD} diverges when $\eta \mathbb E \{\breve \tau\} \left\|\pmb H \right\|_2 > \pi/2$. This suggests the potential that the results for the solvable cases can be extended to general non-convex cases.} 
\begin{figure*}[t]
    \centering
    \subfigure[Time-performance of {ASGD} versus the number of clients with a fixed learning rate.]{\includegraphics[width=3.4in]{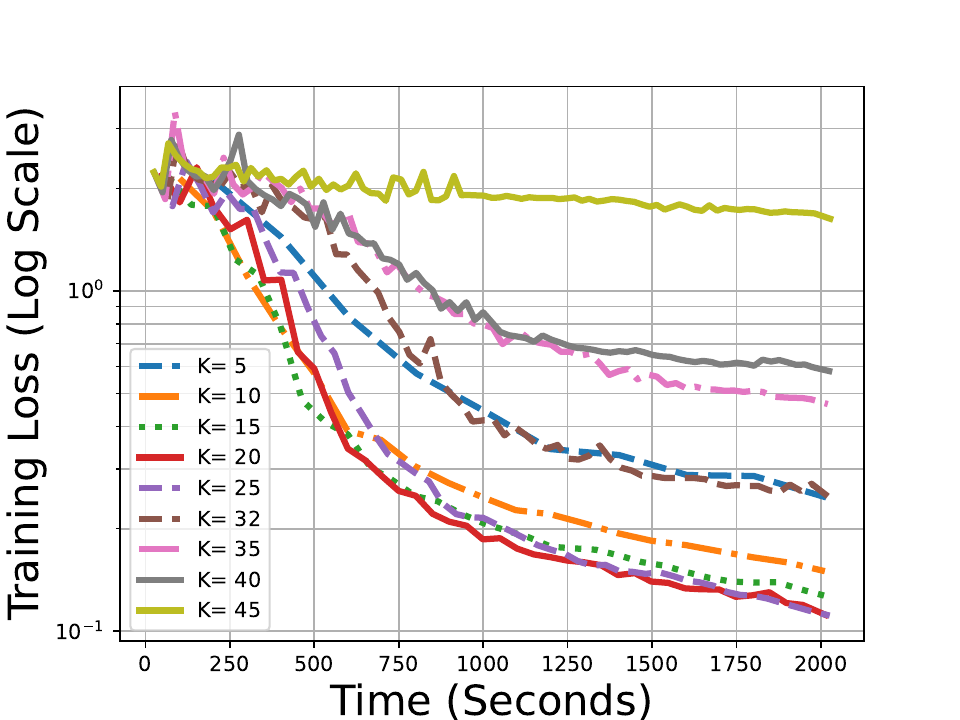}\label{fig_mnist} }
    \hfill
    \subfigure[Time-performance of {ASGD} versus the number of clients with a worker number-aware learning rate.]{\includegraphics[width=3.4in]{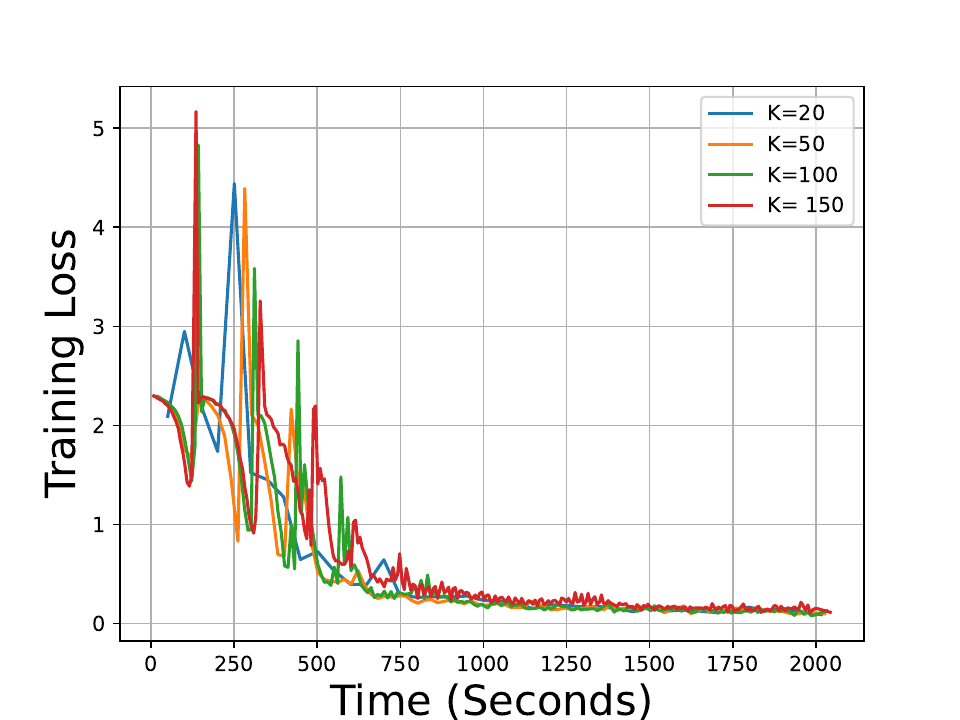}\label{fig_user_aware} }
    \caption{{ASGD} with gradient staleness on the MNIST dataset.}
    \label{fig:realSim}
\end{figure*}

With a fixed step-size, the performance of {ASGD} is shown to deteriorate with a large number of workers, or equivalently, with a large degree of gradient staleness. Next, we aim to evaluate the performance of {ASGD} with a worker number-aware step-size. In particular, we adopt the same setup as the experiment presented in Fig. \ref{fig_mnist}.  Fig. \ref{fig_user_aware} presents the performance of {ASGD} with worker number-aware step size. In particular, the step size is $\eta(K) = \frac{\eta_0}{K}$, wherein $K$ is the number of workers and $\eta_0 = 0.2$ is a constant. In this case, it is seen that the {ASGD} algorithm converges with a worker number-aware step size. However, a careful reader may also notice that the performance of the SGD algorithm does not significantly improve with the increase of $K$ in terms of wall-clock time. This is due to the fact that the positive effect of the increased number of gradients as $K$ increases is offset by the step size inversely proportional to the number of workers.  While asynchronous learning can outperform synchronous learning with respect to the wall-clock time with a relatively small $K$. This observation aligns with the literature, such as \cite{slow_and_stale}.


\section{Conclusion}
In this work, we have presented performance analysis and protocol design criteria for asynchronous SGD with gradient staleness through an SDDE-based approach. To deal with non-exponentially distributed computation time, a Poisson approximation for the superposition of independent renewal processes has been adopted, based on which we have presented the run-time and staleness analysis of {ASGD}. In this way, we have bridged the behavior of the {ASGD} algorithm with gradient staleness and the solution of certain SDDE.
We have derived the convergence condition of distributed SGD by analyzing its characteristic roots, which allows us to optimize scheduling policies for asynchronous distributed SGD. In particular, the characteristic roots have been shown to be closely related to the learning rate, step staleness, and the eigenvalues of the objective function's Hessian matrix. {This work has also demonstrated that ASGD has a higher error floor compared to synchronous SGD. However, this issue can be mitigated by reducing the level of asynchrony as ASGD approaches the local minimum.}
It has been shown that the presence of a small gradient could slightly accelerate the convergence of the SGD, while a large degree of gradient staleness leads to its convergence. It has also been shown that, regardless of the specific distribution of computation time, the expectation of the step staleness in {ASGD} without gradient dropout is determined only by the number of workers and the group size. With limited bandwidth, excessive workers can lead to a large communication delay due to network congestion, degrading the performance of SGD. Further, we have also extended the derivation of the characteristic roots of the case where the gradient staleness in SDDE follows a uniform distribution, which arises in event-triggered SGD.  In practice, for a given learning rate, it has been shown that there exists an optimal number of workers in the {ASGD} tasks, for the fact that a large number of workers leads to a greater degree of staleness which may lead to divergence of the SGD algorithm. 
Numerical simulation results have been presented to demonstrate the potential of our SDDE framework, even in complex learning tasks with non-convex objective functions.

{
ASGD benefits from a higher gradient arrival rate compared to synchronous SGD.  It is noteworthy that the step staleness or discrete-time staleness, rather than the gradient delay, plays a crucial role in determining the convergence rate of {ASGD}. It has been demonstrated that, through careful selection of the number of workers or by designing appropriate protocols, step staleness can be effectively controlled. Also, it is possible to leverage stale gradients to accelerate the convergence of ASGD. Important theoretical potential directions for extending our SDDE-based works involve further analytical results with a non-convex objective function, staleness-aware learning rates, ASGD variants, and non-Gaussian gradient noise. Important practical future directions include developing adaptive learning rates and designing protocols that leverage gradient staleness to accelerate the convergence of ASGD.}

\section*{Acknowledgement}
The authors would like to thank the editor and anonymous reviewers for their professional and constructive comments which helped to improve the quality of this paper.

{
\appendices
\section{Proof of Proposition \ref{proFirstHitting}}\label{appendix:1}
Without loss of generality, we consider the hitting time to $\delta > 0$ rather than  $-\delta$. Hence, in this proof, let $x(0)>0$ such that the OU process first reaches $\delta > 0$. Also, the subscript $i$, which denotes the $i$-th OU process is omitted in this proof for brevity. First, we aim to show that the probability of the OU process leaving the interval \( (-\infty, \delta) \) after it enters this interval can be bounded by a  small positive number \( \epsilon > 0 \).

Let $k_0$ be the solution to $x(k_0) = \delta$. Then it holds that $\mathbb E\{x(k_0 + k\eta)\} = x_0 \exp(-{v} k \eta)$ and $\text{Var}\{x(k_0 +k \eta)\} = \frac{\sigma^2}{2{v}}(1 - \exp(-2 {v} k \eta))$ where $x(k_0 + k\eta)$ is a Gaussian distributed random variable.
Therefore, we have $\Pr(x(k_0 + k\eta)> \delta) \overset{(a)}{\le} Q\left(\frac{\sqrt{2{v}}\delta}{\sigma}(1- \exp(-{v} k\eta)\right)$
where $(a)$ follows from the fact that $1 - \exp(-2{v} k \eta) \le 1$ with positive ${v}, k,$ and $ \eta$.
Note that in practice, the evaluation of the value of the objective variable is only conducted at a few discrete time points. Let $E$ denote the event that $x(k_0 + k\eta) > \delta$ for any $k \in (k_0 + k_l, k_0 + k_r)$, where $k \in \mathbb N^+$, $0<\epsilon, \beta<1$, $k_l =\frac{-\log (\beta)}{{v}\eta}$, and $k_r= \frac{\epsilon}{\eta\exp\left(-\frac{(1-\beta)^2\delta^2 {v}}{\sigma^2}\right)}$. In the following, we aim to demonstrate the fact that 
\begin{equation}\label{add_prob}
    P(E)\le\epsilon. 
\end{equation} 
Clearly, for $k > k_l$, it holds that $\Pr(x(k_0 + k\eta)> \delta)\le \exp\left(- \frac{{v}\delta^2}{\sigma^2}(1-\beta)^2\right)$.
Therefore, 
\begin{align}
    P(E) \le k_r \eta \exp\left(- \frac{{v}\delta^2}{\sigma^2}(1-\beta)^2\right) \le\epsilon.
\end{align}
In this way, it can be seen that once the OU process enters the interval $(-\infty, \delta)$, the probability that the OU process leaves the interval is quite small. Further, we have $\Pr\{{H}\ge t\} =    \Pr\{{H}\ge t|x(t) \ge \delta\}   \Pr\{x(t) \ge \delta\} +  \Pr\{{H}\ge t|x(t) < \delta\}   \Pr\{x(t) < \delta\}.$
Clearly,  $\Pr\{{H}\ge t|x(t) < \delta\} = 0$. Next, it holds that $ \Pr\{{H}\ge t|x(t) \ge \delta\} + \Pr\{{H}< t|x(t) \ge \delta\} = 1$. By Eq. (\ref{add_prob}), we make the approximation that $Pr\{{H}< t|x(t) \ge \delta\}\approx 0$, i.e., $\Pr\{{H}\ge t\} \approx \Pr \{x(t) \ge \delta\}.$    

Therefore,  $\Pr \{x(t) \ge \delta\} = Q\left(\frac{x_0\exp(-{v} t) - \delta}{\sigma_o(t)}\right)$, 
where $Q(\cdot)$ is the $Q$-function and $\sigma_o(t) = \frac{\sigma}{\sqrt{2{v}}} \sqrt{1 - \exp{(-2{v} t)}}$. Let $p_H(t) = - \frac{\rm d }{\rm dt}\Pr \{x(t) \ge \delta\}$, which is the p.d.f. of the first hitting time.  
In this case, we have
\begin{align}\label{pdf_long}
& p_H(t) = \exp\left(-{\left(\frac{x_0 e^{-{v} t} - \delta}{\sqrt{2}\sigma_o(t)}\right)^2}\right)\times 
\\&\left[\frac{-{v} x_0 \exp(-{v} t)}{\sigma_o(t)}  -{\frac{(x_0 \exp(-{v} t) - \delta) \frac{\sigma {v} \exp(-2{v} t)}{\sqrt{2{v}(1 - \exp(-2{v} t))}}}{(\sigma_o(t))^2}} \right].\nonumber
\end{align}
Here, we are interested in the expression of Eq. (\ref{pdf_long}) around $t  = \frac{1}{{v}}\log\frac{x_0}{\delta}$. First, note that for $t = t_0 + \Delta t$, $G(t) = x_0 \exp(-{v} t) - \delta$ is on the order of $\mathcal O(\Delta t)$. This is due to the fact that $x_0 \exp(-{v} t) - \delta = \delta(\exp(-{v}\Delta t) -1)$. Further,  Therefore, $x_0 \exp(-{v} t) - \delta \approx -\delta {v}\Delta t$ by the first order Taylor expansion. Hence, it holds that $\lim_{\Delta_t \to 0} G_2(t_0+ \Delta t)= 0$. By ignoring the contribution of $G(t)$ as $\Delta t \to 0$, we have 
\begin{equation}
    p_H(t_0 + \Delta t) = \frac{{v}\delta \exp(-{v}\Delta t)}{\sigma_o(t_0+\Delta t)}
    \exp\left({-\frac{{v}\delta (\Delta t)^2}{2\sigma^2_o(t_0+\Delta t)}}\right) + \mathcal O(\Delta t).
\end{equation}
By approximating \( \exp(-v \Delta t) \approx 1 \) and \( \sigma_o(t_0 + \Delta t) \approx \sigma_o(t_0) \) for small \( \Delta t \), where the approximation error is on the order of \( \mathcal O(\Delta t) \), we find that
\begin{equation}
    p_H(t_0 + \Delta t) =  \frac{1}{\tilde \sigma \sqrt{2 \pi}} \exp\left(-\frac{(\Delta t)^2}{2 \tilde \sigma^2}\right) + \mathcal O(\Delta t),
\end{equation}
where $\tilde \sigma = \sqrt{\frac{\sigma^2}{2 v_i^3 \delta^2} \left( 1 - \frac{1}{\alpha_i^2} \right)}$.In practice, \( \mathcal O(\Delta t) \) is on the same order of \( \mathcal O(\sigma) \), which suggests that this approximation is accurate, especially with a small gradient noise, i.e., small \( \sigma \). Thus, we obtain Eq. (\ref{approximation}).}
{
\section{A Bound on the Probability of the Event $H_{\max} \ne \tilde H$.}\label{appendix:2}
A careful reader may observe that the $H_{\max}$ is not necessarily equal to $\tilde H$. More specifically, define $i = \arg \max H_j$. where $H_j$ is given by $H_j = \inf \{ t> 0 : |x_j(t) -  x_j^*|< \delta \}$, as defined in Eq. (\ref{eq_def_Hi}). Hence, if $|x_j(H_{\max})|\le \delta$ for all $j \ne i$, then $H_{\max}$ and $\tilde H$ are identical. Otherwise, if $|x_j(H_{\max})|> \delta$ for any $j \ne i$, then we have $H_{\max}\ne \tilde H$.
Let $L_j = |H_j - H_{\max}|$, representing the time difference between $H_j$ and $H_{\max}$. Next, we would like to demonstrate the probability of the event $|x_j(H_{\max})|> \delta$ decays exponentially with $L_j$. Specifically, we can establish that
\begin{align}
\text{Pr}(|x_j(H_{\max})|> \delta) \overset{(a)}{\le} 2\exp\left(- \frac{v_j^3 \delta^2L_j ^2}{\sigma^2}\right),
\end{align}
where $Q(\cdot)$ is the $Q$-function and the inequality $(a)$ holds by the Chernoff bound of the $Q$-function, i.e., $Q(s) \le e^{-\frac{s^2}{2}}$, $s>0$.
Therefore, we conclude that $\text{Pr}(|x_j(H_{\max})|> \delta)$ is on the order of $\mathcal O(e^{-L_j^2})$, and hence it becomes negligible for large values of $L_j$.}

\section{Proof for Corollary \ref{coro_solution_for_delay}}
Let $\omega = \zeta\lambda$. It holds that $\zeta  = \frac{\omega^2}{e^{-\omega}-1}$ by Eq. (\ref{chara_function}). Let us define $g(\omega) = \frac{\omega^2}{e^{-\omega}-1}$. Therefore, the derivative of $g(\omega)$ with respect to $\omega$ is given by 
\begin{equation}
    g'(\omega) = \frac{2\omega(e^{-\omega} -1) + \omega^2e^{-\omega}}{(e^{-\omega}-1)^2}.
\end{equation}  
Therefore, $g(\omega)$ is monotonically increasing in $(-\infty, \varpi)$ and monotonically decreasing in $(\varpi, \infty)$ in which $g'(\varpi) = 0$. Further, it holds that $(\varpi+2)e^{-(\varpi+2)} = 2e^{-2}$. Hence, we have $\varpi = -W_0\left(-\frac2{e^{2}}\right)-2$. Therefore,  the characteristics equation Eq. (\ref{chara_function}) has two distinct negative real roots for $\zeta \in (0, \frac{\varpi ^2}{e^{-\varpi }-1})$.
Further, suppose the solution for Eq. (\ref{diff_eq_uni_delay}) is pure imaginary when $\omega = \gamma$, i.e., $\lambda = i\theta$. In this case, it holds that $\sin(-\gamma\theta) = 0$ and $\cos(-\gamma\theta) = -\gamma \theta^2 + 1.$
Therefore, we have $\gamma = \frac{\pi^2}{2}$ and $\theta = \pm \frac{2}{\pi}$. By this means, we arrive at Corollary \ref{coro_solution_for_delay}.

\section{Proof of Proposition \ref{prop_opt_number_ideal}}\label{app:add}
Given the learning rate $\eta$, the two roots of the characteristics function is given by determined by Proposition \ref{two_roots}. {Further,} it holds that $\text{Re}\{\lambda_0\} > \text{Re}\{\lambda_1\}$ when $\breve \tau \eta v< \frac{1}{e}$, and $\text{Re}\{\lambda_0\} =  \text{Re}\{\lambda_1\}$ when $\breve \tau \eta v\ge  \frac{1}{e}$. {In this way, the convergence rate of the SGD algorithm is mainly determined by the dominant root $\text{Re}\{\lambda_0\}$, i.e.,
$\mathbb E \{\hat x(t) \} = A_0e^{\lambda_0 {  K\nu t}} + A_1e^{\lambda_1 {  K\nu t}}$ where $A_0$ and $A_1$ are constants while $K\nu$ corresponds to the rate of the gradient arrival process. Next, we adopt dominant pole approximation by focusing on the dominant pole $K\nu\lambda_0$, which has the largest effect on the system's response. In particular, it holds that $\lambda_0 =  \frac{W_0{(-(K-1){v}\eta)}}{(K-1)\eta}.$
Such, $K^*= \arg\min_K \frac{K\text{Re}\{W_0(-( K-1){v}\eta)\}}{(K-1)\eta}$. Next, it holds that $\frac{K}{K-1} = 1 + \mathcal O(1/K)$ under the condition that $K \gg 1$. Notably, 
$\text{Re}\{W_0(s)\}$ attains its minimum at $s = -\frac{1}{e}$, leading us to the  proposition.}

\section{Proof of Proposition \ref{prop7}}
By modeling the communication channel with a fixed service rate in pure {ASGD} as an $M/D/1$ queue, the dominant root of the {ASGD} algorithm is given by $\lambda_{\text{Q}} = \frac{W_0{(-\mathbb E \{\breve \tau_\text{Q}\} \eta v)}}{\mathbb E \{\breve \tau_\text{Q}\} \eta }$, in which $\mathbb E \{\breve \tau_\text{Q}\}$ is given by Eq. (\ref{aver_delayMD1}). Therefore, the optimum number of workers in {ASGD} through a queue can be determined by solving the following optimization problem $\min_{K\le K_\text{max}} ~~~\frac{\text{Re}\left\{W_0\left(-K\eta v z(K)\right)\right\}}{\eta  z(K)},$
in which $\check K= \frac{\mu_Q}{\nu}$ is the capacity of the queue and $z(K) =1+ \frac{1}{2\check K} \left(\frac{K}{\check K- K} + 2\right)$. 
Next, the optimal number of workers in this setting, denoted by $K^*_{Q}$, should satisfy that $0< K^*_{Q}\eta v z(K^*_{Q})\le \frac{1}{e} $ such that the $W_0({-K^*_{Q}\eta v z(K^*_{Q})})$ is real since $z(K)$ is monotonically increasing with $K$. Therefore, we could focus on the following optimization problem:
\begin{align}\label{subopt_user}
    &\min_{0<  Kz(K)\le \frac{1}{e\eta v}} ~~~  \frac{W_0\left(-K\eta v z(K)\right)}{z(K)}.
\end{align}
By calculating the first derivative of (\ref{subopt_user}), we conclude that the optimum solution to (\ref{subopt_user}) can be determined by solving $\frac{Kz'(K)}{z(K)+ K z'(K)}  = \frac{1}{1+ W_0(-K\eta v z(K))},$
in which $ z'(K) = \frac{\check K}{(\check K-K)^2}$.
Since  $W_0(s)$ achieves its minimum $W_0(s) = -1$ at $s = {1}/{e}$, the optimal solution $K^*_{Q}$ approximately satisfies $ K^*_{Q}\eta v z(K^*_{Q}) = \frac 1e.$
By this means, we arrive at Eq. (\ref{eq_prop7}).

\end{document}